\documentclass{article}

\usepackage{times}
\usepackage{graphicx} % more modern
\usepackage{subcaption}

\usepackage{natbib}

\usepackage{algorithm}
\usepackage{algorithmic}

\usepackage{hyperref}

\usepackage[accepted]{icml2016}

\usepackage[draft]{fixme}
\fxsetup{inline,nomargin,theme=color}
\FXRegisterAuthor{ufx}{uanfx}{\color{blue}Uri}

\usepackage{amsmath}
\usepackage{amsthm}
\usepackage{amsfonts}
\usepackage{comment}
\usepackage{todonotes}

\newtheorem{thmdef}{Definition}
\newtheorem{thmlem}{Lemma}

\newtheorem{thmthm}{Theorem}
\newtheorem{thmappthm}{Theorem}
\newtheorem{thmasmp}{Assumption}

\newtheorem{thmappdef}{Definition}

\newtheorem{thmappasmp}{Assumption}

\newtheorem{thmapplem}{Lemma}

\def\E{\mathbb{E}}

\def\cX{\mathcal X}
\def\cY{\mathcal Y}
\def\cH{\mathcal H}

\def\cS{\mathcal S}
\def\cF{\mathrm{G}}

\def\cR{\mathcal{R}}
\def \R{\mathbb{R}}

\def \epehe{\epsilon_{\text{PEHE}}}
\def \eate{\epsilon_{\text{ATE}}}
\def \eatt{\epsilon_{\text{ATT}}}

\def \epehenn{{\epehe}_{nn}}

\def\tarnet{{TARNet}}

\newcommand\indep{\protect\mathpalette{\protect\independenT}{\perp}}
\def\independenT#1#2{\mathrel{\rlap{$#1#2$}\mkern2mu{#1#2}}}

\newcommand{\pc}{p^{t=0}}
\newcommand{\pt}{p^{t=1}}

\newcommand{\dPhi}{\frac{\partial \Phi(x)}{\partial x}}
\newcommand{\dPsi}{\frac{\partial \Psi(r)}{\partial r}}

\newcommand{\lyth}{\ell_{h,\Phi}(x,t)}
\newcommand{\lythr}{\ell_{h,\Phi}(\Psi(r),t)}
\newcommand{\lyxzeroh}{\ell_{h,\Phi}(x,0)}
\newcommand{\lyxoneh}{\ell_{h,\Phi}(x,1)}

\newcommand{\lyzerohpsi}{\ell_{h,\Phi}(\Psi(r),0)}
\newcommand{\lyonehpsi}{\ell_{h,\Phi}(\Psi(r),1)}

\newcommand{\Jphix}{\frac{\partial \Phi (x)}{\partial x}}
\newcommand{\GP}{\Gamma_\Phi}

\icmltitlerunning{Estimating individual treatment effect: generalization bounds and algorithms}

\begin{document}

\twocolumn[
\icmltitle{Estimating individual treatment effect: generalization bounds and algorithms}

\icmlauthor{Uri Shalit*}{shalit@cs.nyu.edu}
\icmladdress{CIMS,
            New York University, New York, NY 10003}
\icmlauthor{Fredrik D. Johansson*}{fredrikj@mit.edu}
\icmladdress{IMES, MIT, Cambridge, MA 02142}
\icmlauthor{David Sontag}{dsontag@csail.mit.edu}
\icmladdress{CSAIL \& IMES, MIT, Cambridge, MA 02139}
%\icmladdress{$^*$ Equal contribution}

\icmlkeywords{counterfactual inference, causal effects}

%\vskip 0.3in
]

\begin{abstract}
There is intense interest in applying machine learning to problems of causal inference in fields such as healthcare, economics and education. In particular, individual-level causal inference has important applications such as precision medicine. We give a new theoretical analysis and family of algorithms for predicting individual treatment effect (ITE) from observational data, under the assumption known as strong ignorability. The algorithms learn a ``balanced'' representation such that the induced treated and control distributions look similar. We give a novel, simple and intuitive generalization-error bound showing that the expected ITE estimation error of a representation is bounded by a sum of the standard generalization-error of that representation and the distance between the treated and control distributions induced by the representation. We use Integral Probability Metrics to measure distances between distributions, deriving explicit bounds for the Wasserstein and Maximum Mean Discrepancy (MMD) distances. Experiments on real and simulated data show the new algorithms match or outperform the state-of-the-art.
\end{abstract}

\section{Introduction}\label{sec:intro}

Making predictions about causal effects of actions is a central problem in many domains. For example, a doctor deciding which medication will cause better outcomes for a patient; a government deciding who would benefit most from subsidized job training; or a teacher deciding which study program would most benefit a specific student. In this paper we focus on the problem of making these predictions based on \emph{observational data}. Observational data is data which contains past actions, their outcomes, and possibly more context, but without direct access to the mechanism which gave rise to the action. For example we might have access to records of patients (context), their medications (actions), and outcomes, but we do not have complete knowledge of why a specific action was applied to a patient.

The hallmark of learning from observational data is that the actions observed in the data depend on variables which might also affect the outcome, resulting in \emph{confounding}: For example, richer patients might better afford certain medications, and job training might only be given to those motivated enough to seek it. The challenge is how to untangle these confounding factors and make valid predictions. Specifically, we work under the common simplifying assumption of ``no-hidden confounding'', assuming that all the factors determining which actions were taken are observed. In the examples above, it would mean that we have measured a patient's wealth or an employee's motivation.

As a learning problem, estimating causal effects from observational data is different from classic learning in that in our training data we never see the individual-level effect. For each unit, we only see their response to one of the possible actions - the one they had actually received. This is close to what is known in the machine learning literature as ``learning from logged bandit feedback'' \citep{strehl2010learning,swaminathan2015batch}, with the distinction that we do not have access to the model generating the action.

Our work differs from much work in causal inference in that we focus on the individual-level causal effect (also known as ``c-specific treatment effects'' \citet{shpitser2006ident,pearl2015detecting}), rather that the average or population level. Our main contribution is to give what is, to the best of our knowledge, the first generalization-error\footnote{Our use of the term generalization is different from its use in the study of \emph{transportability}, where the goal is to generalize causal conclusion across distributions \citep{bareinboim2016causal}.} bound for estimating individual-level causal effect, where each individual is identified by its features $x$. The bound leads naturally to a new family of representation-learning based algorithms \cite{bengio2013representation}, which we show to match or outperform state-of-the-art methods on several causal effect inference tasks.

We frame our results using the Rubin-Neyman potential outcomes framework \cite{rubin2011causal}, as follows. We assume that for a unit with features $x \in \cX$, and an action (also known as treatment or intervention) $t \in \{0,1\}$, there are two potential outcomes: $Y_0$ and $Y_1$. In our data, for each unit we only see one of the potential outcomes, depending on the treatment assignment: if $t=0$ we observe $y=Y_0$, if $t=1$, we observe $y=Y_1$; this is known as the \emph{Consistency} assumption. For example, $x$ can denote the set of lab tests and demographic factors of a diabetic patient, $t=0$ denote the standard medication for controlling blood sugar, $t=1$ denotes a new medication, and $Y_0$ and $Y_1$ indicate the patient's blood sugar level if they were to be given medications $t=0$ and $t=1$, respectively.

We will denote $m_1(x) = \E\left[Y_1|x\right]$, $m_0(x) = \E\left[Y_0|x\right]$.
We are interested in learning the function $\tau(x) := \E \left[Y_1 -Y_0|x\right] = m_1(x) - m_0(x)$. $\tau(x)$ is the expected \emph{treatment effect} of $t=1$ relative to $t=0$ on an individual unit with characteristics $x$, or the Individual Treatment Effect (ITE) \footnote{Sometimes known as the Conditional Average Treatment Effect, CATE.}. For example, for a patient with features $x$, we can use this to predict which of two treatments will have a better outcome. The fundamental problem of causal inference is that for any $x$ in our data we only observe $Y_1$ or $Y_0$, but never both. %In this paper we are mainly interested in observational data, i.e. the case where the distribution of the treatment assignment $t$ is dependent on $x$. %For example, richer patients might better afford different medications, and job training might only be given to those motivated enough to seek it.

As mentioned above, we make an important ``no-hidden confounders'' assumption, in order to make the conditional causal effect identifiable. We formalize this assumption by using the standard \emph{strong ignorability} condition: $(Y_1, Y_0)\indep t | x $, and $0<p(t=1|x)<1$ for all $x$. Strong ignorability is a sufficient condition for the ITE function $\tau(x)$ to be identifiable \cite{imbens2009recent,pearl2015detecting,rolling2014estimation}: see proof in the supplement. The validity of strong ignorability cannot be assessed from data, and must be determined by domain knowledge and understanding of the causal relationships between the variables.

One approach to the problem of estimating the function $\tau(x)$ is by learning the two functions $m_0(x)$ and $m_1(x)$ using samples from $p(Y_t|x,t)$. This is similar to a standard machine learning problem of learning from finite samples. However, there is an additional source of variance at work here: For example, if mostly rich patients received treatment $t=1$, and mostly poor patients received treatment $t=0$, we might have an unreliable estimation of $m_1(x)$ for poor patients. %This is very similar to the phenomena of covariate shift \cite{mansour2009bdomain,johansson2016counterfactual}: the expected labeling functions $m_t(x)$ are fixed, but we need to perform inference over a different underlying distribution $p(x,t)$ from the one in the training set.
In this paper we upper bound this additional source of variance using an Integral Probability Metric (IPM) measure of distance between two distributions $p(x|t=0)$, and $p(x|t=1)$, also known as the \emph{control} and \emph{treated} distributions. %The IPM is calibrated by a notion of complexity related to the functions $m_t(x)$, and thus scales naturally with the hardness of inference.
In practice we use two specific IPMs: the Maximum Mean Discrepancy \cite{gretton2012mmd}, and the Wasserstein distance \cite{villani2008optimal,cuturi2014fast}. We show that the expected error in learning the individual treatment effect function $\tau(x)$ is upper bounded by the error of learning $Y_1$ and $Y_0$, plus the IPM term. In the randomized controlled trial setting, where $t \indep x$, the IPM term is $0$, and our bound naturally reduces to a standard learning problem of learning two functions.

The bound we derive points the way to a family of algorithms based on the idea of representation learning \cite{bengio2013representation}: Jointly learn  hypotheses for both treated and control on top of a representation which minimizes a weighted sum of the factual loss (the standard supervised machine learning objective), and the IPM distance between the control and treated distributions induced by the representation. This can be viewed as learning the functions $m_0$ and $m_1$ under a constraint that encourages better generalization across the treated and control populations.
In the Experiments section we apply algorithms based on multi-layer neural nets as representations and hypotheses, along with MMD or Wasserstein distributional distances over the representation layer; see Figure \ref{fig:neuralnet} for the basic architecture.

%A somewhat similar architecture to the one we use was proposed by \citet{johansson2016counterfactual}. The algorithms we propose are conceptually simpler than the previous ones, since they do not require a two-stage fitting procedure, and avoid the need to compute nearest neighbors; we also show in the Experiments section that our new method outperforms our previous one. We also provide generalization theory and experiments on real-world datasets, including out-of-sample causal inference, which are not present in the above paper.

%We also show that our methods achieve competitive results on a real-world causal inference benchmark: the widely used National Supposed Work survey \cite{lalonde1986evaluating,smith2005does}.

In his foundational text about causality, \citet{pearl2009causality} writes: ``Whereas in traditional learning tasks we attempt to generalize from one set of instances to another, the causal modeling task is to generalize from behavior under one set of conditions to behavior under another set. \emph{Causal models should therefore be chosen by a criterion that challenges their stability against changing conditions...}'' [emphasis ours]. We believe our work points the way to one such stability criterion, for causal inference in the strongly ignorable case.

%Our contributions in this paper are as follows:
%(1) We prove a novel, simple and meaningful upper bound on the expected ITE error. (2) We present a family of algorithms which are applicable to a wide set of models for classification and regression. (3) Our experiments show the benefit of our method, and we show results on a real-world causal inference benchmark, the National Work Survey dataset \cite{lalonde1986evaluating,smith2005does}.

\begin{figure}[t!]
  \centering
  \includegraphics[width=1.0\columnwidth]{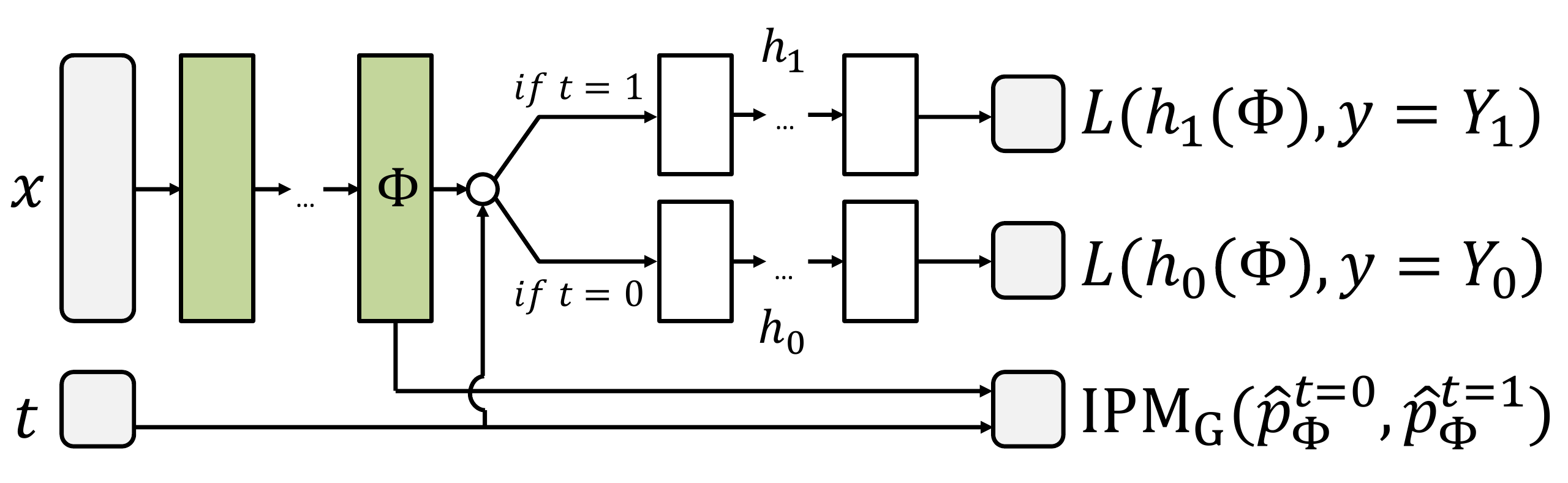}
  \caption{\label{fig:neuralnet}Neural network architecture for ITE estimation. $L$ is a loss function, $\text{IPM}_\cF$ is an integral probability metric. Note that only one of $h_0$ and $h_1$ is updated for each sample during training. }
\end{figure}

\section{Related work}

Much recent work in machine learning for causal inference focuses on \emph{causal discovery}, with the goal of discovering the underlying causal graph or causal direction from data \citep{hoyer2009nonlinear,maathuis2010predicting,triantafillou2015constraint,mooij2016distinguishing}. We focus on the case when the causal graph is simple and known to be of the form $(Y_1,Y_0) \leftarrow x \rightarrow t$, with no hidden confounders.

Under the causal model we assume, the most common goal of causal effect inference as used in the applied sciences is to obtain the average treatment effect: $ATE=\E_{x \sim p(x)}\left[\tau(x)\right]$. We will briefly discuss how some standard statistical causal effect inference methods relate to our proposed method. Note that most of these approaches assume some form of ignorability.

One of the most widely used approaches to estimating ATE is covariate adjustment, also known as back-door adjustment or the G-computation formula \citep{pearl2009causality,rubin2011causal}. In its basic version, covariate adjustment amounts to estimating the functions $m_1(x)$, $m_0(x)$. Therefore, covariate adjustment methods are the most natural candidates for estimating ITE as well as ATE, using the estimates of $m_t(x)$. However, most previous work on this subject focused on asymptotic consistency \citep{belloni2014inference,athey2016efficient,chernozhukov2016double}, and so far there has not been much work on the generalization-error of such a procedure. One way to view our results is that we point out a previously unaccounted for source of variance when using covariate adjustment to estimate ITE. We suggest a new type of regularization, by learning representations with reduced IPM distance between treated and control, enabling a new type of bias-variance trade-off.

Another widely used family of statistical methods used in causal effect inference are weighting methods. Methods such as propensity score weighting \citep{austin2011introduction} re-weight the units in the observational data so as to make the treated and control populations more comparable. These methods do not yield themselves immediately to estimating an individual level effect, and adapting them for that purpose is an interesting research question.
%Matching methods match treated and control units according to one of many criteria (see  \cite{stuart2010matching}). Since matching can be seen as covariate adjustment with nearest neighbor functions, from our theoretical perspective they are closely related to covariate adjustment.
Doubly robust methods combine re-weighting the samples and covariate adjustment in clever ways to reduce model bias \citep{funk2011doubly}. Again, we believe that finding how to adapt the concept of double robustness to the problem of effectively estimating ITE is an interesting open question.

%Causal inference: consistency
Adapting machine learning methods for causal effect inference, and in particular for individual level treatment effect, has gained much interest recently. For example \citet{wager2015estimation,athey2016recursive} discuss how tree-based methods can be adapted to obtain a consistent estimator with semi-parametric asymptotic convergence rate. %Others show how to adapt high-dimensional regression methods such as Lasso to consistently estimate treatment effect, again achieving semi-parametric rates \citep{belloni2014inference,athey2016efficient,chernozhukov2016double}.
Recent work has also looked into how machine learning method can help detect heterogeneous treatment effects when some data from randomized experiments is available \citep{taddy2016nonparametric,peysakhovich2016combining}.
Neural nets have also been used for this purpose, exemplified in early work by \citet{beck2000improving}, and more recently by \citet{hartford2016counterfactual}'s work on deep instrumental variables.
%\citet{hartford2016counterfactual} use two neural networks to estimate both the propensity of treatment and the relationship between covariates, treatment and outcome, but rely on the existence of instrumental variables (variables that cause solely the treatment and not the outcome), which are often not available.
Our work differs from all the above by focusing on the generalization-error aspects of estimating individual treatment effect, as opposed to asymptotic consistency, and by focusing solely on the observational study case, with no randomized components or instrumental variables.

%Causal inference: bounds
Another line of work in the causal inference community relates to bounding the estimate of the average treatment effect given an instrumental variable \citep{balke1997bounds,bareinboim2012controlling}, or under hidden confounding, for example when the ignorability assumption does not hold \citep{pearl2009causality,cai2008bounds}. Our work differs, in that we only deal with the ignorable case, and in that we bound a very different quantity: the generalization-error of estimating individual level treatment effect.

%Domain adaptation
Our work has strong connections with work on domain adaptation. In particular, estimating ITE requires prediction of outcomes over a different distribution from the observed one. Our ITE error upper bound has similarities with generalization bounds in domain adaptation given by \citet{ben2007analysis,mansour2009bdomain,bendavid2010theory,cortes2014domain}.
These bounds employ distribution distance metrics such as the A-distance or the discrepancy metric, which are related to the IPM distance we use. Our algorithm is similar to a recent algorithm for domain adaptation by \citet{JMLR:v17:15-239}, and in principle other domain adaptation methods (e.g. \citet{daume2009frustratingly,pan2011domain,sun2016return}) could be adapted for use in ITE estimation as presented here.

%Previous paper
Finally, our paper builds on work by \citet{johansson2016counterfactual}, where the authors show a connection between covariate shift and the task of estimating the counterfactual outcome in a causal inference scenario. They proposed learning a representation of the data that makes the treated and control distributions more similar, and fitting a linear ridge-regression model on top of it. They then bounded the relative error of fitting a ridge-regression using the distribution with reverse treatment  assignment versus fitting a ridge-regression using the factual distribution. Unfortunately, the relative error bound is not at all informative regarding the absolute quality of the representation. In this paper we focus on a related but more substantive task: estimating the individual treatment effect, building on top of the counterfactual error term. We further provide an informative bound on the absolute quality of the representation. We also derive a much more flexible family of algorithms, including non-linear hypotheses and much more powerful distribution metrics in the form of IPMs such as the Wasserstein and MMD distances. Finally, we conduct significantly more thorough experiments including a real-world dataset and out-of-sample performance, and show our methods outperform previously proposed ones.

\section{Estimating ITE: Error bounds}\label{sec:theory}

In this section we prove a bound on the expected error in estimating the individual treatment effect for a given representation, and a hypothesis defined over that representation.
The bound is expressed in terms of (1) the expected loss of the model when learning the observed outcomes $y$ as a function of $x$ and $t$, denoted $\epsilon_F$, $F$ standing for ``Factual''; (2) an Integral Probability Metric (IPM) distance between the distribution of treated and control units. The term $\epsilon_F$ is the classic machine learning generalization-error, and in turn can be upper bounded using the empirical error and model complexity terms, applying standard machine learning theory \citep{shalev2014understanding}.

\subsection{Problem setup}
We will employ the following assumptions and notations. The most important notations are in the Notation box in the supplement.
The space of covariates is a bounded subset $\cX \subset \R^d$. The outcome space is $\cY\subset \R$. Treatment $t$ is a binary variable. We assume there exists a joint distribution $p(x,t,Y_0,Y_1)$, such that $(Y_1,Y_0) \indep t |x$ and $0<p(t=1|x)<1$ for all $x\in \cX$ (strong ignorability).
%We call the marginal of $p$ over $(x,t)$ the \emph{factual distribution}, and denote it $p^F(x,t)$.
The treated and control distributions are the distribution of the features $x$ conditioned on treatment: $\pt(x) := p(x|t=1)$, and $\pc(x) := p(x|t=0)$, respectively.

%\begin{thmdef}\label{def:fcf}
%Let the counterfactual distribution $p^{CF}(x,t)$ be the distribution over $\cX \times \{0,1\}$ such that $p^{CF}(1-t,x) := p^{F}(t,x)$, where $\cX \subset \R^d$.
%\end{thmdef}

%The counterfactual distribution is the factual distribution with the treatment assignment flipped.
%Note that if treatment assignment is independent of the context $x$, and $p^F(t|x) = \frac{1}{2}$, then the counterfactual and factual distributions are identical; this is exactly the case of a randomized controlled trial, where treatment is assigned randomly with equal probability to units.

Throughout this paper we will discuss \emph{representation functions} of the form
 $\Phi : \cX \rightarrow \cR$, where $\cR$ is the representation space. We make the following assumption about $\Phi$:
\begin{thmasmp}\label{asmp:inv}
The representation $\Phi$ is a twice-differentiable, one-to-one function. Without loss of generality we will assume that $\cR$ is the image of $\cX$ under $\Phi$. We then have $\Psi :\cR \rightarrow \cX$ as the inverse of $\Phi$, such that $\Psi(\Phi(x)) = x$ for all $x \in \cX$.
\end{thmasmp}
The representation $\Phi$ pushes forward the treated and control distributions into the new space $\cR$; we denote the induced distribution by $p_\Phi$.
\begin{thmdef}\label{def:phit}
Define $\pt_\Phi(r) := p_\Phi(r|t=1)$, $\pc_\Phi(r) := p_\Phi(r|t=0)$, to be the treated and control distributions induced over $\cR$.
For a one-to-one $\Phi$, the distributions $\pt_\Phi(r)$ and $\pc_\Phi(r)$ can be obtained by the standard change of variables formula, using the determinant of the Jacobian of $\Psi(r)$.
\end{thmdef} %See \cite{ben1999change} for the case of a mapping $\Phi$ between spaces of different dimensions.

Let $\Phi : \cX \rightarrow \cR$ be a representation function, and $h :\cR \times \{0,1\} \rightarrow \cY$ be an hypothesis defined over the representation space $\cR$. Let $L: \cY \times \cY \rightarrow \R_+$ be a loss function. We define two complimentary loss functions: one is the standard machine learning loss, which we will call the factual loss and denote $\epsilon_F$. The other is the expected loss with respect to the distribution where the treatment assignment is flipped, which we call the counterfactual loss, $\epsilon_{CF}$.

\begin{thmdef}\label{def:perunitloss}
The expected loss for the unit and treatment pair $(x,t)$ is:
$\ell_{h,\Phi}(x,t) = \int_\cY L(Y_t,h(\Phi(x),t)) p(Y_t|x) dY_t.$
The expected factual and counterfactual losses of $h$ and $\Phi$ are:
\begin{align*}
&\epsilon_F(h,\Phi) = \int_{\cX \times \{0,1\}} \!\!\!\!\!\!\! \lyth\, p(x,t)\, dxdt, \\
&\epsilon_{CF}(h,\Phi) = \int_{\cX \times \{0,1\}} \!\!\!\!\!\!\! \lyth\, p(x,1-t)\, dxdt .
\end{align*}
\end{thmdef}

If $x$ denotes patients' features, $t$ a treatment, and $Y_t$ a potential outcome such as mortality, we think of $\epsilon_F$ as measuring how well do $h$ and $\Phi$ predict mortality for the patients and doctors' actions sampled from the same distribution as our data sample. $\epsilon_{CF}$ measures how well our prediction with $h$ and $\Phi$ would do in a ``topsy-turvy'' world where the patients are the same but the doctors are inclined to prescribe exactly the opposite treatment than the one the real-world doctors would prescribe.

\begin{thmdef}\label{def:decompef}
The expected factual \emph{treated} and \emph{control} losses are:
\begin{align*}
&\epsilon^{t=1}_F(h,\Phi) = \int_{\cX } \!\!\! \lyxoneh\, \pt(x)\, dx ,\\
&\epsilon^{t=0}_F(h,\Phi) = \int_{\cX } \!\!\! \lyxzeroh\, \pc(x)\, dx .
\end{align*}
\end{thmdef}
For $u:=p(t=1)$, it is immediate to show that $\epsilon_F(h,\Phi) = u \epsilon^{t=1}(h,\Phi) + (1-u) \epsilon^{t=0}(h,\Phi)$.

\begin{thmdef}\label{def:te}
The treatment effect (ITE) for unit $x$ is:
$$\tau(x) := \E\left[Y_1 - Y_0 | x \right].$$
\end{thmdef}
%\begin{thmdef}\label{def:m}
%For $t=0,1$ define:
%$$m_t(x) := \E\left[Y_t|x\right].$$
%\end{thmdef}
%Obviously $\tau(x) = m_1(x)-m_0(x)$.

Let $f: \cX \times \{0,1\} \rightarrow \cY$ by an hypothesis. For example, we could have that $f(x,t) = h(\Phi(x),t)$.
\begin{thmdef}\label{def:inditeerr}
The treatment effect estimate of the hypothesis $f$ for unit $x$ is:
\begin{align*}
\hat{\tau}_f  (x) = f(x,1) - f(x,0).
\end{align*}
\end{thmdef}

\begin{thmdef}
The expected Precision in Estimation of Heterogeneous Effect (PEHE,  \citet{hill2011bayesian}) loss of $f$ is:
\begin{equation}
\label{eq:pehe}
\epehe(f) = \int_{\cX } \left( \hat{\tau}_f(x) - \tau(x) \right)^2 \, p(x) \, dx ,
\end{equation}
When $f(x,t) = h(\Phi(x),t)$, we will also use the notation $\epehe(h,\Phi) = \epehe(f)$.
\end{thmdef}

Our proof relies on the notion of an \emph{Integral Probability Metric} (IPM), which is a class of metrics between probability distributions \citep{sriperumbudur2012empirical,muller1997integral}.
For two probability density functions $p$, $q$ defined over $\cS \subseteq \R^d$, and for a function family $\cF$ of functions $g : \cS \rightarrow \R$, we have that $$\text{IPM}_\cF(p,q) :=\sup_{g\in \cF} \left| \int_\cS g(s) (p(s)-q(s))\, ds \right|.$$
Integral probability metrics are always symmetric and obey the triangle inequality, and trivially satisfy $\text{IPM}_\cF(p,p) = 0$. For rich enough function families $\cF$, we also have that $\text{IPM}_\cF(p,q) =0 \implies p=q,$ and then $\text{IPM}_\cF$ is a true metric over the corresponding set of probabilities. Examples of function families $\cF$ for which $\text{IPM}_\cF$ is a true metric are the family of bounded continuous functions, the family of $1$-Lipschitz functions \citep{sriperumbudur2012empirical}, and the unit-ball of functions in a universal reproducing Hilbert kernel space \citep{gretton2012mmd}.

\begin{thmdef}
Recall that $m_t(x) = \E\left[Y_t|x\right]$.
The expected variance of $Y_t$ with respect to a distribution $p(x,t)$:
$$\sigma^2_{Y_t}(p(x,t)) = \int_{\cX \times \cY} \left(Y_t - m_t(x)\right)^2 p(Y_t|x)p(x,t) \, dY_t dx .$$
We define:
\begin{align*}
&\sigma^2_{Y_t} = \min\{\sigma^2_{Y_t}(p(x,t)), \sigma^2_{Y_t}(p(x,1-t))\} , \\
&\sigma^2_{Y} = \min\{\sigma^2_{Y_0},\sigma^2_{Y_1}\}.
\end{align*}
\end{thmdef}
%If $Y_t$ are deterministic functions of $x$, then $\sigma^2_{Y} = 0$.
\subsection{Bounds}

We first state a Lemma bounding the counterfactual loss, a key step in obtaining the bound on the error in estimating individual treatment effect. We then give the main Thoerem. The proofs and details are in the supplement.

Let $u := p(t=1)$ be the marginal probability of treatment. By the strong ignorability assumption, $0<u<1$.

\begin{thmlem}\label{lem:gen}
Let $\Phi : \cX \rightarrow \cR$ be a one-to-one representation function, with inverse $\Psi$.
Let $h : \cR \times \{0,1\} \rightarrow \cY$ be an hypothesis. Let $\cF$ be a family of functions $g: \cR \rightarrow \cY$. Assume there exists a constant $B_\Phi>0$, such that for fixed $t \in \{0,1\}$, the per-unit expected loss functions $\lythr$ (Definition \ref{def:perunitloss}) obey $ \frac{1}{B_\Phi} \cdot  \lythr \in \cF$. We have:
%For the hypothesis $f: \cX \times \{0,1\} \rightarrow \cY$ such that $f(x,t) = h(\Phi(x),t)$:
\begin{align*}
&\epsilon_{CF}(h,\Phi) \leq \nonumber \\
&\quad (1-u)  \epsilon^{t=1}_F(h,\Phi) + u  \epsilon^{t=0}_F(h,\Phi) \nonumber \\
&  \quad + B_\Phi \cdot \text{IPM}_\cF\left( \pt_\Phi, \pc_\Phi \right),
\end{align*}
where $\epsilon_{CF}$, $ \epsilon^{t=0}_F$ and $ \epsilon^{t=1}_F$ are as in Definitions \ref{def:perunitloss} and \ref{def:decompef}.
\end{thmlem}

\begin{thmthm}\label{thm:gen}
Under the conditions of Lemma \ref{lem:gen}, and assuming the loss $L$ used to define $\ell_{h,\Phi}$ in Definitions \ref{def:perunitloss} and \ref{def:decompef} is the squared loss, we have:
\begin{align}
&\epehe(h,\Phi) \leq  \nonumber \\
& 2\!\left(\epsilon_{CF}(h,\Phi) + \epsilon_F(h,\Phi) - 2\sigma^2_Y \right) \leq  \label{eq:thm_gen} \\
&2\!\left(\epsilon_F^{t=0}(h,\Phi)\! +\!\epsilon_F^{t=1}(h,\Phi)\!+ \! B_\Phi  \text{IPM}_\cF \left( \pt_\Phi, \pc_\Phi \right)\! -\! 2\sigma^2_Y\right)\!, \nonumber
\end{align}
where $\epsilon_F$ and $\epsilon_{CF}$ are defined w.r.t. the squared loss.
\end{thmthm}

The main idea of the proof is showing that $\epehe$ is upper bounded by the sum of the expected factual loss $\epsilon_F$ and expected counterfactual loss $\epsilon_{CF}$.  However, we cannot estimate $\epsilon_{CF}$, since we only have samples relevant to $\epsilon_F$.  We therefore bound the difference $\epsilon_{CF} - \epsilon_F$ using an IPM.

Choosing a small function family $\cF$ will make the bound tighter. However, choosing too small a family could result in an incomputable bound. For example, for the minimal choice $\cF = \{\lyxzeroh, \lyxoneh  \}$, we will have to evaluate an expectation term of $Y_1$ over $\pc_\Phi$, and of $Y_0$ over $\pt_\Phi$. We cannot in general evaluate these expectations, since by assumption when $t=0$ we only observe $Y_0$, and the same for $t=1$ and $Y_1$. In addition, for some function families there is no known way to efficiently compute the IPM distance or its gradients. In this paper we use two function families for which there are available optimization tools. The first is the family of $1$-Lipschitz functions, which leads to IPM being the Wasserstein distance \citep{villani2008optimal,sriperumbudur2012empirical}, denoted $\text{Wass}(p,q)$. The second is the family of norm-$1$ reproducing kernel Hilbert space (RKHS) functions, leading to the MMD metric \citep{gretton2012mmd,sriperumbudur2012empirical}, denoted $\text{MMD}(p,q)$. Both the Wasserstein and MMD metrics have consistent estimators which can be efficiently computed in the finite sample case \citep{sriperumbudur2012empirical}. Both have been used for various machine learning tasks in recent years \citep{gretton2009covariate,gretton2012mmd,cuturi2014fast}.

In order to explicitly evaluate the constant $B_\Phi$ in Theorem \ref{thm:gen}, we have to make some assumptions about the elements of the problem. For the Wasserstein case these are the loss $L$, the Lipschitz constants of $p(Y_t|x)$ and $h$, and the condition number of the Jacobian of $\Phi$. For the MMD case, we make assumptions about the RKHS representability and RKHS norms of $h$ , $\Phi$, and the standard deviation of $Y_t|x$. The full details are given in the supplement, with the major results stated in Theorems 2 and 3. In all cases we obtain that making $\Phi$ smaller increases the constant $B_\Phi$ precluding trivial solutions such as making $\Phi$ arbitrarily small.

For an empirical sample, and a \emph{family} of representations and hypotheses, we can further upper bound $\epsilon_F^{t=0}$ and $\epsilon_F^{t=1}$ by their respective empirical losses and a model complexity term using standard arguments \citep{shalev2014understanding}. The IPMs we use can be consistently estimated from finite samples \citep{sriperumbudur2012empirical}. The negative variance term $\sigma^2_Y$ arises from the fact that, following \citet{hill2011bayesian,athey2016recursive}, we define the error $\epehe$ in terms of the conditional mean functions $m_t(x)$, as opposed to fitting the random variables $Y_t$.

Our results hold for any given $h$ and $\Phi$ obeying the Theorem conditions. This immediately suggest an algorithm in which we minimize the upper bound in Eq. \eqref{eq:thm_gen} with respect to $\Phi$ and $h$ and either the Wasserstein or MMD IPM, in order to minimize the error in estimating the individual treatment effect. This leads us to Algorithm \ref{alg:model} below.

\section{Algorithm for estimating ITE}\label{sec:model}

We propose a general framework called CFR (for Counterfactual Regression) for ITE estimation based on the theoretical results above. Our algorithm is an end-to-end, regularized minimization procedure which simultaneously fits both a balanced representation of the data and a hypothesis for the outcome.
CFR draws on the same intuition as the approach proposed by  \citet{johansson2016counterfactual}, but \emph{overcomes} the following limitations of their method: a) Their theory requires a two-step optimization procedure and is specific to \emph{linear} hypotheses of the learned representation (and does not support e.g. deep neural networks), b) The treatment indicator might get lost if the learned representation is high-dimensional (see discussion below).

%We note that our framework is also more flexible in practice, as our theory supports multiple measures of balance that can can be minimized efficiently; this is only rarely true for variants of the discrepancy distance used by \cite{johansson2016counterfactual}.

We assume there exists a distribution $p(x,t,Y_0,Y_1)$ over $\cX \times \{0,1\} \times \cY \times \cY$, such that strong ignorability holds. We further assume we have a sample from that distribution $(x_1,t_1,y_1), \dots (x_n,t_n,y_n)$, where $y_i \sim p(Y_1|x_i)$ if $t_i=1$, $y_i \sim p(Y_0|x_i)$ if $t_i=0$. This standard assumption means that the treatment assignment determines which potential outcome we see. Our goal is to find a representation $\Phi : \cX \rightarrow \cR$ and hypothesis $h:  \cX \times \{0,1\} \rightarrow \cY$ that will minimize $\epehe(f)$ for $f(x,t):=h(\Phi(x),t)$.

In this work, we let $\Phi(x)$ and $h(\Phi,t)$ be parameterized by deep neural networks trained jointly in an end-to-end fashion, see Figure~\ref{fig:neuralnet}. This model allows for learning complex non-linear representations and hypotheses with large flexibility. \citet{johansson2016counterfactual} parameterized $h(\Phi,t)$ with a single network using the concatenation of $\Phi$ and $t$ as input. When the dimension of $\Phi$ is high, this risks losing the influence of $t$ on $h$ during training. To combat this, our first contribution is to parameterize $h_1(\Phi)$ and $h_0(\Phi)$ as two separate ``heads'' of the joint network, the former used to estimate the outcome under treatment, and the latter under control. This means that statistical power is shared in the representation layers of the network, while the effect of treatment is retained in the separate heads. Note that each sample is used to update only the head corresponding to the observed treatment; for example, an observation $(x_i, t_i=1, y_i)$ is only used to update $h_1$.

Our second contribution is to excplicitly account and adjust for the bias induced by treatment group imbalance. To this end, we seek a representation $\Phi$ and hypothesis $h$ that minimizes a trade-off between predictive accuracy and imbalance in the representation space, using the following objective:
\begin{equation}
\begin{aligned}
\min_{\substack{h, \Phi\\ \|\Phi\|=1}} & \;\; \frac 1 n \sum_{i=1}^n w_i \cdot L\left(h(\Phi(x_i),t_i)\, ,y_i\right) + \lambda \cdot \mathfrak{R}(h)  \\
& \;\; + \alpha  \cdot \text{IPM}_\cF\left(\{\Phi(x_i)\}_{i:t_i=0},\{\Phi(x_i)\}_{i:t_i=1}\right), \\
\mbox{with} & \;\; w_i = \frac{t_i}{2u} + \frac{1-t_i}{2(1-u)},\;\;\mbox{where}\;\;u = \frac{1}{n}\sum_{i=1}^n t_i, \\
\mbox{and} & \;\;\mbox{$\mathfrak{R}$ is a model complexity term.}
\label{eq:overallalg}
\end{aligned}
\end{equation}
Note that $u=p(t=1)$ in the definition of $w_i$ is simply the proportion of treated units in the population. The weights $w_i$ compensate for the difference in treatment group size in our sample, see Theorem~\ref{thm:gen}. $\text{IPM}_{\cF}(\cdot,\cdot)$ is the (empirical) integral probability metric defined by the function family $\cF$. For most IPMs, we cannot compute the factor $B_\phi$ in Equation~\ref{eq:thm_gen}, but treat it as part of the hyperparameter $\alpha$. This makes our objective sensitive to the scaling of $\Phi$, even for a constant $\alpha$. We therefore normalize $\Phi$ through either projection or batch-normalization with fixed scale. We refer to the model minimizing \eqref{eq:overallalg} with $\alpha>0$ as Counterfactual Regression (CFR) and the variant without balance regularization ($\alpha=0$) as Treatment-Agnostic Representation Network (\tarnet{}).

We train our models by minimizing \eqref{eq:overallalg} using stochastic gradient descent, where we backpropagate the error through both the hypothesis and representation networks, as described in Algorithm~\ref{alg:model}. Both the prediction loss and the penalty term $\text{IPM}_{\cF}(\cdot,\cdot)$ are computed for one mini-batch at a time. Details of how to obtain the gradient $g_1$ with respect to the empirical IPMs are in the supplement.

\begin{algorithm}[tbp]
\caption{CFR: Counterfactual regression with integral probability metrics}
\label{alg:model}
\begin{algorithmic}[1]
  \STATE \textbf{Input:} Factual sample $(x_1,t_1,y_1), \ldots , (x_n,t_n,y_n)$, scaling parameter $\alpha>0$, loss function $L\left(\cdot,\cdot\right)$, representation network $\Phi_{\bf{W}}$ with initial weights $\bf{W}$, outcome network $h_{\bf{V}}$ with initial weights $\bf{V}$, function family $\cF$ for IPM.
\STATE Compute $u=\frac{1}{n}\sum_{i=1}^n t_i$
 \STATE Compute $w_i = \frac{t_i}{2u} + \frac{1-t_i}{2(1-u)}$ for $i=1\ldots n$
  \WHILE{not converged}
   % \STATE Sample a mini-batch $\{(x_{i_j},t_{i_j},y_{i_j})\}_{j=1}^m$
       \STATE Sample mini-batch  $\{i_1,i_2,\ldots, i_m\} \subset  \{1,2,\ldots, n\}$
    \STATE Calculate the gradient of the IPM term:\\
    $g_1 = ${\small $\nabla_{\bf{W}} \; \text{IPM}_{\cF}(\{\Phi_{\bf{W}}(x_{i_j})\}_{t_{i_j} = 0}, \{\Phi_{\bf{W}}(x_{i_{k}})\}_{t_{i_j} = 1} )$}
    \STATE Calculate the gradients of the empirical loss: \\
    $g_2 = \nabla_{\bf{V}}\frac{1}{m} \sum_j w_{i_j} \cdot L\left(h_{\bf{V}}(\Phi_{\bf{W}}(x_{i_j}),t_{i_j}), y_{i_j}\right)$ \\
    $g_3 = \nabla_{\bf{W}}\frac{1}{m}\sum_j w_{i_j} \cdot L\left(h_{\bf{V}}(\Phi_{\bf{W}}(x_{i_j}),t_{i_j}), y_{i_j}\right)$
    \STATE Obtain step size scalar or matrix $\eta$ with standard neural net methods e.g. Adam~\citep{kingma2014adam}
    %\STATE Update $\mathbf{W} \leftarrow \mathbf{W} - \eta(\alpha g_1 +  g_3 )$
    %\STATE Update $\mathbf{V} \leftarrow \mathbf{V} - \eta(g_2 + 2 \lambda \mathbf{V})$
    \STATE  $\left[\mathbf{W},\mathbf{V}\right] \leftarrow \left[\mathbf{W} - \eta(\alpha g_1 +  g_3 ),  \mathbf{V} - \eta(g_2 + 2 \lambda \mathbf{V}) \right]$
    \STATE Check convergence criterion
  \ENDWHILE
\end{algorithmic}
\end{algorithm}

\section{Experiments}
\label{sec:experiments}

Evaluating causal inference algorithms is more difficult than many machine learning tasks, since for real-world data we rarely have access to the ground truth treatment effect. Existing literature mostly deals with this in two ways. One is by using synthetic or semi-synthetic datasets, where the outcome or treatment assignment are fully known; we use the semi-synthetic IHDP dataset from \citet{hill2011bayesian}. The other is using real-world data from randomized controlled trials (RCT). The problem in using data from RCTs is that there is no imbalance between the treated and control distributions, making our method redundant. We partially overcome this problem by using the Jobs dataset from \citet{lalonde1986evaluating}, which includes both a randomized and a non-randomized component. We use both for training, but can only use the randomized component for evaluation. This alleviates, but does not solve, the issue of a completely balanced dataset being unsuited for our method.

We evaluate our framework CFR, and its variant without balancing regularization (\tarnet{}), in the task of estimating ITE and ATE. CFR is implemented as a feed-forward neural network with 3 fully-connected exponential-linear layers for the representation and 3 for the hypothesis. Layer sizes were 200 for all layers used for Jobs and 200 and 100 for the representation and hypothesis used for IHDP. The model is trained using Adam~\citep{kingma2014adam}. For an overview, see Figure~\ref{fig:neuralnet}. Layers corresponding to the hypothesis are regularized with a small $\ell_2$ weight decay.  For continuous data we use mean squared loss and for binary data, we use log-loss. While our theory does not immediately apply to log-loss, we were curious to see how our model performs with it. %sWe use the Wasserstein (CFR Wass) and MMD (CFR MMD) distances to penalize imbalance.

We compare our method to Ordinary Least Squares with treatment as a feature (OLS-1), OLS with separate regressors for each treatment (OLS-2), $k$-nearest neighbor ($k$-NN), Targeted Maximum Likelihood, which is a doubly robust method (TMLE)~\citep{gruber2011tmle}, Bayesian Additive Regression Trees (BART)~\citep{chipman2010bart,bayestree}, Random Forests (Rand. For.)~\citep{breiman2001random}, Causal Forests  (Caus. For.)~\citep{wager2015estimation} as well as the Balancing Linear Regression (BLR) and Balancing Neural Network (BNN) by \citet{johansson2016counterfactual}.
%We also compare to a variable selection procedure dubbed LASSO + Ridge (L+R) in which a ridge regression model is fit to the variables selected by LASSO.
For classification tasks we substitute Logistic Regression (LR) for OLS.  %$\ell_1$-regularized Logistic Regression ($\ell_1$-LR) for OLS and L+R respectively.
Choosing hyperparameters for estimating PEHE is non-trivial; we detail our selection procedure, applied to all methods, in subsection C.1 of the supplement.

%\paragraph{Hyperparameter selection}\vskip -13pt
%Standard methods for hyperparameter selection, such as cross-validation, are not generally applicable for estimating the PEHE loss since only one potential outcome is observed (unless the outcome is simulated). For real-world data, we may use the observed outcome $y_{j(i)}$ of the nearest neighbor $j(i)$ to $i$ in the opposite treatment group, $t_{j(i)} = 1 - t_i$ as surrogate for the counterfactual outcome. We use this to define a nearest-neighbor approximation of the PEHE loss, $\epehenn(f) = \frac{1}{n}\sum_{i=1}^n \left((1-2t_i)(y_{j(i)} - y_i) - (f(x_i,1) - f(x_i,0))\right)^2~$. On IHDP, we use the objective value on the validation set for early stopping in CFR, and $\epehenn(f)$ for hyperparameter selection. On the Jobs dataset, we use the policy risk on the validation set (see Section~\ref{sec:nsw}).

We evaluate our model in two different settings. One is \emph{within-sample}, where the task is to estimate ITE for all units in a sample for which the (factual) outcome of \emph{one} treatment is observed. This corresponds to the common scenario in which a cohort is selected once and not changed. This task is non-trivial, as we never observe the ITE for any unit. The other is the \emph{out-of-sample} setting, where the goal is to estimate ITE for units with no observed outcomes. This corresponds to the case where a new patient arrives and the goal is to select the best possible treatment. Within-sample error is computed over both the training and validation sets, and out-of-sample error over the test set.

\vskip -7pt
\subsection{Simulated outcome: IHDP}\vskip -5pt
\citet{hill2011bayesian} compiled a dataset for causal effect estimation based on the Infant Health and Development Program (IHDP), in which the covariates come from a randomized experiment studying the effects of specialist home visits on future cognitive test scores. The treatment groups have been made imbalanced by removing a biased subset of the treated population. The dataset comprises 747 units (139 treated, 608 control) and 25 covariates measuring aspects of children and their mothers. We use the simulated outcome implemented as setting ``A'' in the NPCI package~\citep{npci}. Following \citet{hill2011bayesian}, we use the \emph{noiseless} outcome to compute the true effect. We report the estimated (finite-sample) PEHE loss $\epehe$ (Eq. ~\ref{eq:pehe}), and the absolute error in average treatment effect $\eate = |\frac{1}{n}\sum_{i=1}^n (f(x_i, 1) - f(x_i, 0)) - \frac{1}{n}\sum_{i=1}^n (m_1(x_i) - m_0(x_i))|$.
The results of the experiments on IHDP are presented in Table~\ref{tbl:ihdp_jobs_results} (left). We average over 1000 realizations of the outcomes with 63/27/10 train/validation/test splits.

\begin{table}[t!]
  \caption{\label{tbl:ihdp_jobs_results}Results on IHDP (left) and Jobs (right). MMD is squared linear MMD. Lower is better. \vspace{-1.3em}}
  \begin{center}
    \begin{small}
      \begin{sc}
      \begin{tabular}{l|cc|cc}
        \multicolumn{5}{l}{\bf{Within-sample}} \\
        \hline
        & \multicolumn{2}{c|}{IHDP} & \multicolumn{2}{c}{Jobs} \\
        & $\sqrt{\epehe}$ & $\eate$ & $R_{\text{Pol}}$  & $\eatt$\\
        \hline
        OLS/LR-1 & $5.8 \pm .3$ & $.73 \pm .04$ & $.22 \pm .0$ & $.01 \pm .00$ \\
        OLS/LR-2  & $2.4 \pm .1$ & $.14 \pm .01$ & $.21 \pm .0$ & $.01 \pm .01$ \\
        BLR  & $5.8 \pm .3$ & $.72 \pm .04$ & $.22 \pm .0$ & $.01 \pm .01$ \\
        $k$-NN & $2.1 \pm .1$ & $.14 \pm .01$ & $.02 \pm .0$ & $.21 \pm .01$ \\
		    TMLE  & $5.0 \pm .2$ & $.30 \pm .01$ & $.22 \pm .0$ & $.02 \pm .01$ \\
		    %L + R / $\ell_1$-LR  & & & & \\
        BART  & $2.1 \pm .1$ & $.23 \pm .01$ & $.23 \pm .0$ & $.02 \pm .00$ \\
        Rand.For.  & $4.2 \pm .2$ & $.73 \pm .05$ & $.23 \pm .0$ & $.03 \pm .01$ \\
        Caus.For.  & $3.8 \pm .2$ & $.18 \pm .01$ & $.19 \pm .0$ & $.03 \pm .01$ \\
        BNN  & $2.2 \pm .1$ & $.37 \pm .03$ & $.20 \pm .0$ & $.04 \pm .01$ \\
        %\hline
        \tarnet{}  & $.88 \pm .0$ & $.26 \pm .01$ & $.17 \pm .0$ & $.05 \pm .02$ \\
        CFR MMD & $.73 \pm .0$ & $.30 \pm .01$ & $.18 \pm .0$ & $.04 \pm .01$ \\
        CFR Wass & $.71 \pm .0$ & $.25 \pm .01$ & $.17 \pm .0$ & $.04 \pm .01$ \\
        \hline
        \multicolumn{5}{l}{\bf{Out-of-sample}} \\
        \hline
        & \multicolumn{2}{c|}{IHDP} & \multicolumn{2}{c}{Jobs} \\
        & $\sqrt{\epehe}$ & $\eate$ & $R_{\text{Pol}}$  & $\eatt$\\
        \hline
        OLS/LR-1  & $5.8 \pm .3$ & $.94 \pm .06$ & $.23 \pm .0$ & $.08 \pm .04$ \\
        OLS/LR-2  & $2.5 \pm .1$ & $.31 \pm .02$ & $.24 \pm .0$ & $.08 \pm .03$ \\
        BLR  & $5.8 \pm .3$ & $.93 \pm .05$ & $.25 \pm  .0$ & $.08 \pm .03$ \\
        $k$-NN & $4.1 \pm .2$ & $.79 \pm .05$ & $.26 \pm .0$ & $.13 \pm .05$ \\
		    %L + R / $\ell_1$-LR  & & & & \\
        BART  & $2.3 \pm .1$ & $.34 \pm .02$ & $.25 \pm .0$ & $.08 \pm .03$ \\
        Rand.For.  & $6.6 \pm .3$ & $.96 \pm .06$ & $.28 \pm .0$ & $.09 \pm .04$ \\
        Caus.For.  & $3.8 \pm .2$ & $.40 \pm .03$ & $.20 \pm .0$ & $.07 \pm .03$ \\
        %BNN-4-0  & & & & \\
        BNN  & $2.1 \pm .1$ & $.42 \pm .03$ & $.24 \pm .0$ & $.09 \pm .04$ \\
        %\hline
        \tarnet{}  & $.95 \pm .0$ & $.28 \pm .01$ & $.21 \pm .0$ & $.11 \pm .04$ \\
        CFR MMD & $.78 \pm .0$ & $.31 \pm .01$ & $.21 \pm .0$ & $.08 \pm .03$ \\
        CFR Wass & $.76 \pm .0$ & $.27 \pm .01$ & $.21 \pm .0$ &  $.09 \pm .03$ \\
        \hline
      \end{tabular}
    \end{sc}
    \end{small}
  \end{center}
  \vspace{-1em}
\end{table}

\begin{figure}[t]
  \centering
  \includegraphics[width=0.85\columnwidth]{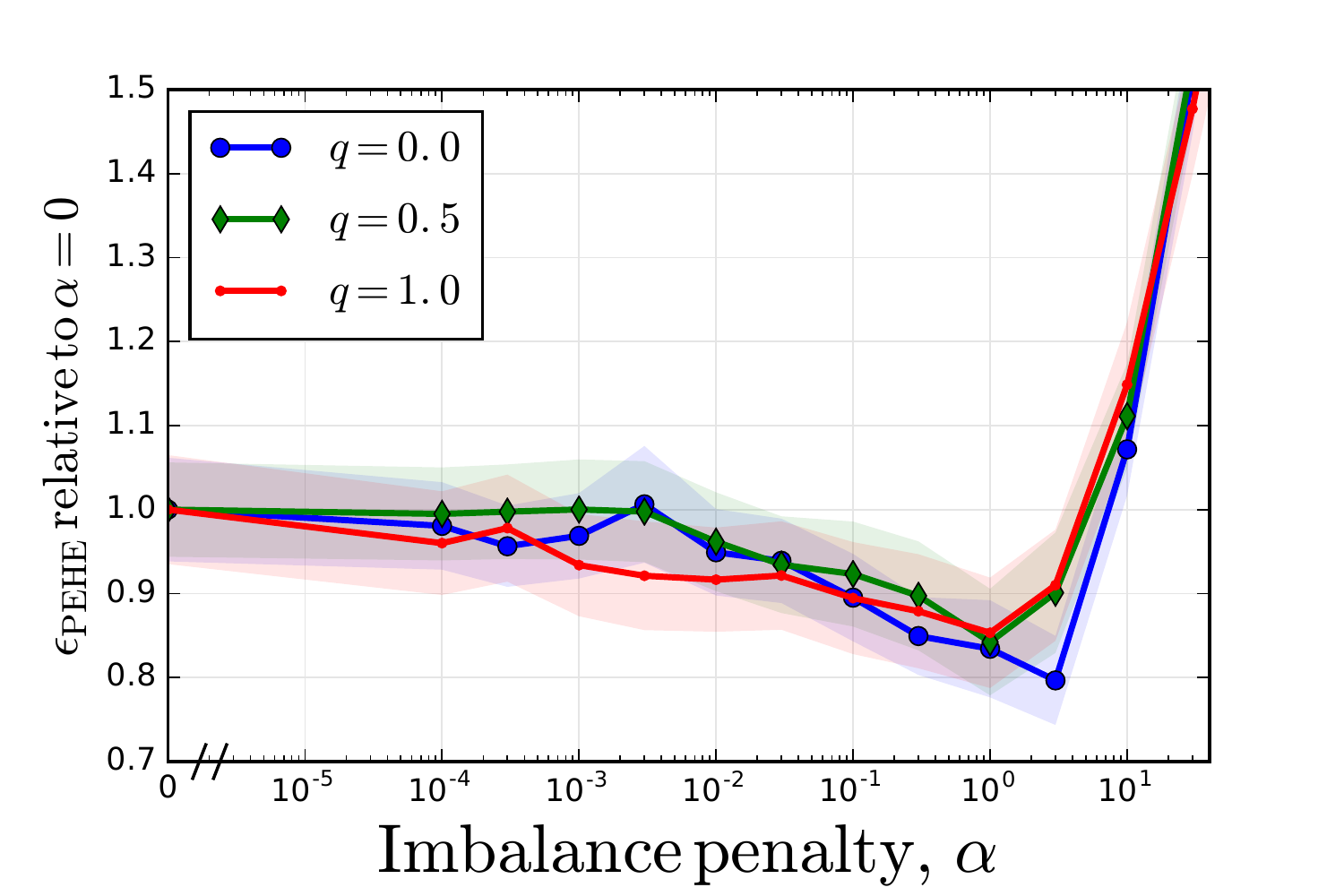}\vspace{-1em}
  \caption{\label{fig:ihdp_rel_imb}Out-of-sample ITE error versus IPM regularization for CFR Wass, relative to the error at $\alpha=0$, on 500 realizations of IHDP, with high ($q=1$), medium and low (artificial) imbalance between control and treated.}
  \vspace{-1.2em}
\end{figure}

%\paragraph{Increasing imbalance}
\vskip -5pt
We investigate the effects of increasing imbalance between the original treatment groups by constructing biased subsamples of the IHDP dataset. A logistic-regression propensity score model is fit to form estimates $\hat{p}(t=1|x)$ of the conditional treatment probability. Then, repeatedly, with probability $q$ we remove the remaining \emph{control} observation $x$ that has $\hat{p}(t=1|x)$ closest to $1$, and with probability $1-q$, we remove a random control observation. The higher $q$, the more imbalance. For each value of $q$, we remove $347$ observations from each set, leaving $400$.

\begin{figure}[t]
  \centering
  \includegraphics[width=0.85\columnwidth]{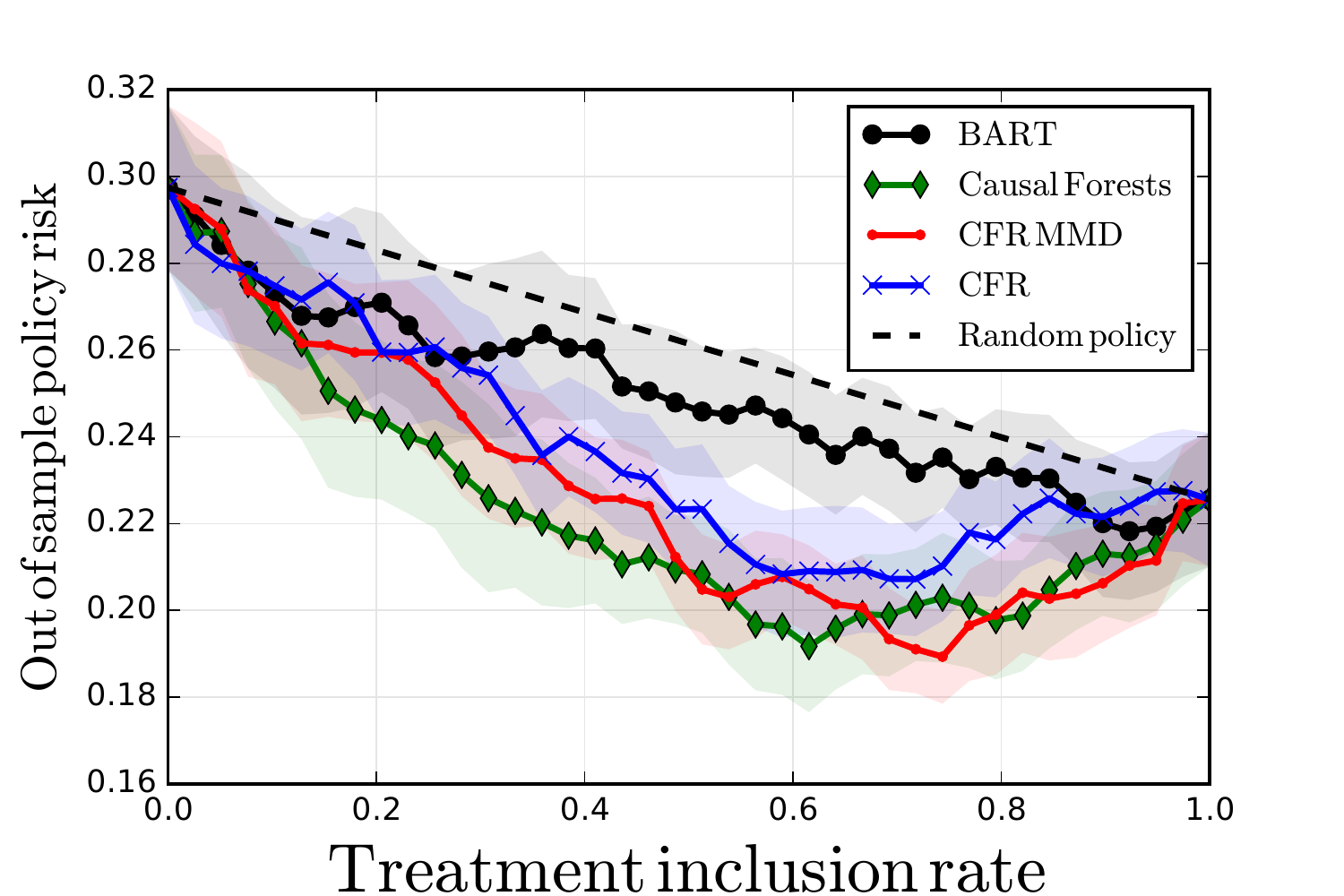}\vspace{-0.5em}
  \caption{\label{fig:policy_curve}Policy risk on Jobs as a function of treatment inclusion rate. Lower is better. Subjects are included in treatment in order of their estimated treatment effect given by the various methods. CFR Wass is similar to CFR and is omitted to avoid clutter. }
  \vspace{-1.2em}
\end{figure}

\subsection{Real-world outcome: Jobs}
\label{sec:nsw}
The study by \citet{lalonde1986evaluating} is a widely used benchmark in the causal inference community, where the treatment is job training and the outcomes are income and employment status after training. This dataset combines a randomized study based on the National Supported Work program with observational data to form a larger dataset~\citep{smith2005does}. The presence of the randomized subgroup gives a way to estimate the ``ground truth'' causal effect. The study includes 8 covariates such as age and education, as well as previous earnings. We construct a \emph{binary} classification task, called \emph{Jobs}, where the goal is to predict unemployment, using the feature set of \citet{dehejia2002propensity}. Following \citet{smith2005does}, we use the LaLonde experimental sample (297 treated, 425 control) and the PSID comparison group (2490 control). There were 482 (15\%) subjects unemployed by the end of the study. We average over 10 train/validation/test splits with ratios 56/24/20.

Because all the treated subjects $T$ were part of the original randomized sample $E$, we can compute the true average treatment effect on the treated by $\text{ATT} = {|T|}^{-1}\sum_{i \in T} y_i - {|C \cap E|}^{-1}\sum_{i \in C \cap E} y_i$, where $C$ is the control group. We report the error $\eatt = | \text{ATT} - \frac{1}{|T|}\sum_{i\in T} (f(x_i, 1) - f(x_i, 0))|$.
We cannot evaluate $\epehe$ on this dataset, since there is no ground truth for the ITE. Instead, in order to evaluate the quality of ITE estimation, we use a measure we call \emph{policy risk}. The policy risk is defined as the average loss in value when treating according to the policy implied by an ITE estimator. In our case, for a model $f$, we let the policy be to treat, $\pi_f(x) = 1$, if $f(x,1) - f(x,0) > \lambda$, and to not treat, $\pi_f(x) = 0$ otherwise. The policy risk is $R_{\text{Pol}}(\pi_f) = 1 - (\mathbb{E}[Y_1|\pi_f(x)=1]\cdot p(\pi_f=1) + \mathbb{E}[Y_0|\pi_f(x)=0]\cdot p(\pi_f=0))$ which we can estimate for the randomized trial subset of Jobs by $\hat{R}_{\text{Pol}}(\pi_f = 1 - (\mathbb{E}[Y_1|\pi_f(x)=1, t=1]\cdot p(\pi_f=1) + \mathbb{E}[Y_0|\pi_f(x)=0, t=0]\cdot p(\pi_f=0))$. See figure~\ref{fig:policy_curve} for risk as a function of treatment threshold $\lambda$, aligned by proportion of treated, and Table \ref{tbl:ihdp_jobs_results} for the risk when $\lambda=0$.
\vskip -5pt
\subsection{Results}
We begin by noting that indeed imbalance confers an advantage to using the IPM regularization term, as our theoretical results indicate, see e.g. the results for CFR Wass ($\alpha > 0$) and \tarnet{} ($\alpha = 0$) on IHDP in Table~\ref{tbl:ihdp_jobs_results}. We also see in Figure~\ref{fig:ihdp_rel_imb} that even for the harder case of increased imbalance ($q>0$) between treated and control, the relative gain from using our method remains significant. On Jobs, we see a smaller gain from using IPM penalties than on IHDP. We believe this is the case because, while we are minimizing our bound over observational data and accounting for this bias, we are evaluating the predictions only on a randomized subset, where the treatment groups are distributed identically.
For both IHDP, non-linear estimators do significantly better than linear ones in terms of individual effect ($\epehe$). On the Jobs dataset, straightforward logistic regression does remarkably well in estimating the ATT. However, being a linear model, LR can only ascribe a uniform policy - in this case, ``treat everyone''. The more nuanced policies offered by non-linear methods achieve lower policy risk in the case of Causal Forests and CFR. This emphasizes the fact that estimating average effect and individual effect can require different models. Specifically, while smoothing over many units may yield a good ATE estimate, this might significantly hurt ITE estimation. $k$-nearest neighbors has very good within-sample results on Jobs, because evaluation is performed over the randomized component, but suffers heavily in generalizing out of sample, as expected.

%
%The results of the IHDP and binary Jobs experiments are presented in Table~\ref{tbl:ihdp_jobs_results}. Non-linear estimators such as BART, CFR and Causal Forests fare  best in terms of the individual error $\epehe$ on IHDP, and using neural networks even without balancing gives good results as well. Our proposed method CFR performs well, achieving better results than linear methods, and is competitive with C.Forests and BART. We can also see a gain from the imbalance penalty $\alpha > 0$, both for $\epehe$ on IHDP and $R_{\text{Policy}}$ on Jobs. Based on the Jobs results, we see that a small average error is not always indicative of a good policy. Logistic regression for example, achieves an impressive estimate of average effect, but can only deliver a uniform policy, as the model is linear in the treatment variable.

\vskip -10pt
\section{Conclusion}
In this paper we give a meaningful and intuitive error bound for the problem of estimating individual treatment effect. Our bound relates ITE estimation to the classic machine learning problem of learning from finite samples, along with methods for measuring distributional distances from finite samples. The bound lends itself naturally to the creation of learning algorithms; we focus on using neural nets as representations and hypotheses. We apply our theory-guided approach to both synthetic and real-world tasks, showing that in every case our method matches or outperforms the state-of-the-art. Important open questions are theoretical considerations in choosing the IPM weight $\alpha$, how to best derive confidence intervals for our model's predictions, and how to integrate our work with more complicated causal models such as those with hidden confounding or instrumental variables.%, as well as exploring in more depth the connection to domain adaptation problems. 

\subsubsection*{Acknowledgments}
We wish to thank Aahlad Manas for his assistance with the experiments. We also thank Jennifer Hill, Marco Cuturi, Esteban Tabak and Sanjong Misra for fruitful conversations, and Stefan Wager for his help with the code for Causal Forests. DS and US were supported by NSF CAREER award \#1350965.

%\section*{References}
{\small
\bibliographystyle{icml2016}
\bibliography{cfr}
}
\newpage

\appendix

\section{Proofs}\label{sec:proof}

\subsection{Definitions, assumptions, and auxiliary lemmas}
\begin{figure*}
\begin{minipage}{\textwidth}%
\noindent\fbox{%
    \parbox{\textwidth}
    {{\bf{Notation:}} \\
     $p(x,t)$: distribution on $\cX \times \{0,1\}$\\
     $u = p(t=1)$: the marginal probability of treatment.\\
     $\pt(x) = p(x|t=1)$: treated distribution. $\pc(x) = p(x|t=0)$: control distribution.\\
     $\Phi$: representation function mapping from $\cX$ to  $\cR$.\\
     $\Psi$: the inverse function of $\Phi$, mapping from $\cR$ to $\cX$.\\
     $p_\Phi(r,t)$: the distribution induced by $\Phi$ on $\cR \times \{0,1\}$.\\
       $\pt_\Phi(r)$, $\pc_\Phi(r)$: treated and control distributions induced by $\Phi$ on $\cR$.\\
     $L(\cdot,\cdot)$: loss function, from $\cY \times \cY$ to $\R_+$.\\
     $\lyth$: the expected loss of $h(\Phi(x),t)$ for the unit $x$ and treatment $t$.\\
     $\epsilon_F(h,\Phi)$, $\epsilon_{CF}(h,\Phi)$: expected factual and counterfactual loss of $h(\Phi(x),t)$.\\
     $\tau(x) := \E\left[Y_1 - Y_0 | x \right]$, the expected treatment effect for unit $x$.\\
     $\epehe(f)$: expected error in estimating the individual treatment effect of a function $f(x,t)$. \\
     $\text{IPM}_\cF(p,q)$: the integral probability metric distance induced by function family $\cF$ between distributions $p$ and $q$.
      %$\text{UIPM}_\cF(u_p \cdot p,u_q \cdot q)$: the unnormalized integral probability metric distance induced by function family $\cF$ between distributions $p$ and $q$ scaled respectively by $u_p, u_q \in \R_+$
   \vskip 0pt
    }}
\end{minipage}
\end{figure*}

We first define the necessary distributions and prove some simple results about them. We assume a joint distribution function $p(x,t,Y_0,Y_1)$, such that $(Y_1, Y_0)\indep t | x $, and $0<p(t=1|x)<1$ for all $x$. 
Recall that we assume \emph{Consistency}, that is we assume that we observe $y=Y_1|(t=1)$ and $y=Y_0|(t=0)$.

\begin{thmappdef}\label{def:teA}
The treatment effect for unit $x$ is:
$$\tau(x) := \E\left[Y_1 - Y_0 | x \right].$$
\end{thmappdef}

We first show that under consistency and strong ignorability, the ITE function $\tau(x)$ is identifiable:
\begin{thmapplem}
We have:
\begin{align}
&\E\left[Y_1-Y_0|x\right] =\nonumber \\
&\E\left[Y_1 |x\right]-\E\left[Y_0|x\right] =  \label{eq:useig}\\
&\E\left[Y_1|x,t=1\right]-\E\left[Y_0|x, t=0\right] = \label{eq:useconsistency}\\
&\E\left[y|x,t=1\right]-\E\left[y|x, t=0\right] .\nonumber 
\end{align}
Equality \eqref{eq:useig} is because we assume that $Y_t$ and $t$ are independent conditioned on $x$.
Equality \eqref{eq:useconsistency} follows from the consistency assumption.
Finally, the last equation is composed entirely of observable quantities and can be estimated from data since we assume $0<p(t=1|x)<1$ for all $x$.
\end{thmapplem}
%Recall that we call the marginal distribution over $\cX \times \{0,1\}$ the \emph{factual distribution}, denoted $p^F(x,t)$.
%\begin{thmappdef}\label{def:fcfA}
%Let the counterfactual distribution $p^{CF}(x,t)$ be the distribution over $\cX \times \{0,1\}$ such that $p^{CF}(1-t,x) := p^{F}(t,x)$, where $\cX \subset \R^d$.
%\end{thmappdef}
%An immediate consequence is that
%\begin{equation}\label{covshiftA}
%p^F(x) = p^{CF}(x).
%\end{equation}

\begin{thmappdef}
Let $\pt(x) := p(x|t=1)$, and $\pc(x) := p(x|t=0)$ denote respectively the treatment and control distributions.
\end{thmappdef}

Let $\Phi : \cX \rightarrow \cR$ be a \emph{representation function}. We will assume that $\Phi$ is differentiable.

\begin{thmappasmp}\label{asmp:invA}
The representation function $\Phi $ is one-to-one. Without loss of generality we will assume that $\cR$ is the image of $\cX$ under $\Phi$, and define $\Psi :\cR \rightarrow \cX$ to be the inverse of $\Phi$, such that $\Psi(\Phi(x)) = x$ for all $x \in \cX$.
\end{thmappasmp}

\begin{thmappdef}\label{def:pphiA}
For a representation function $\Phi :\cX \rightarrow \cR$, and for a distribution $p$ defined over $\cX$, let $p_\Phi$ be the distribution induced by $\Phi$ over $\cR$. Define $\pt_\Phi(r) := p_\Phi(r|t=1)$, $\pc_\Phi(r) := p_\Phi(r|t=0)$, to be the treatment and control distributions induced over $\cR$.
\end{thmappdef}

For a one-to-one $\Phi$, the distribution $p_\Phi$ over $\cR \times \{0,1\}$ can be obtained by the standard change of variables formula, using the determinant of the Jacobian of $\Psi(r)$. See \cite{ben1999change} for the case of a mapping $\Phi$ between spaces of different dimensions.

%\begin{thmapplem}\label{lemma:smpl_alg1}
%$p^F_\Phi(r) = p^{CF}_\Phi(r)$
%\end{thmapplem}
%\begin{proof}
%Let $J_\Psi(r)$ be the absolute of the determinant of the Jacobian of $\Psi(r)$. Then by the change of variable formula we have:
%$$p^F_\Phi(r) = J_\Psi(r) \cdot p^F(\Psi(r)) = J_\Psi(r) \cdot p^{CF}(\Psi(r)) = p^{CF}_\Phi(r),$$
%since by \eqref{covshiftA} we have $p^{F}(\Psi(r)) = p^{CF}(\Psi(r))$.
%\end{proof}

\begin{thmapplem}\label{lemma:smpl_alg3}
For all $r \in \cR$, $t \in \{0,1\}$:
\begin{align*}
&p_\Phi(t|r) = p(t|\Psi(r))\\
& p(Y_t|r) = p(Y_t|\Psi(r)).
\end{align*}
\end{thmapplem}
\begin{proof}
Let $J_\Psi(r)$ be the absolute of the determinant of the Jacobian of $\Psi(r)$.
\begin{align*}
&p_\Phi(t|r) = \frac{p_\Phi(t,r)}{p_\Phi(r)}  \stackrel{(a)}{=}
\frac{p(t,\Psi(r))J_\Psi(r)}{p(\Psi(r)) J_\Psi(r)} = \\
&\frac{p(t,\Psi(r))}{p(\Psi(r))} = p(t|\Psi(r)),
\end{align*}
where equality (a) is by the change of variable formula. The proof is identical for $p(Y_t|r)$.
\end{proof}

%\begin{thmapplem}\label{lemma:smpl_alg2}
%For all $r \in \cR$:
%$$p^{CF}_\Phi(r,t) = p^F_\Phi(r|1-t) p^F_\Phi(1-t)$$
%\end{thmapplem}
%\begin{proof}
%$$p^{CF}_\Phi(r,t) = p^{CF}_\Phi(r) p^{CF}_\Phi(t|r) = p^{F}_\Phi(r) p^{F}_\Phi(1-t|r),$$
%since by Lemma \ref{lemma:smpl_alg1} we have $p^{CF}_\Phi(r) = p^{F}_\Phi(r)$, and by Definition \ref{def:fcfA} and Lemma \ref{lemma:smpl_alg3}
%\end{proof}

%%%%%%%%%%%%%%%%%%%%%%%%

Let $L : \cY \times \cY \rightarrow \R_+$ be a loss function, e.g. the absolute loss or squared loss.
\begin{thmappdef}\label{def:perunitlossA}
Let $\Phi : \cX \rightarrow \cR$ be a representation function.
Let $h :\cR \times \{0,1\} \rightarrow \cY$ be an hypothesis defined over the representation space $\cR$.
The expected loss for the unit and treatment pair $(x,t)$ is:
$$\lyth = \int_\cY L(Y_t,h(\Phi(x),t)) p(Y_t|x) dY_t$$
\end{thmappdef}

\begin{thmappdef}\label{def:epfepcf}
The expected factual loss and counterfactual losses of $h$ and $\Phi$ are, respectively:
$$\epsilon_F(h,\Phi) = \int_{\cX \times \{0,1\}} \!\!\!\!\!\!\! \lyth\, p(x,t)\, dxdt $$
$$\epsilon_{CF}(h,\Phi) = \int_{\cX \times \{0,1\}} \!\!\!\!\!\!\! \lyth\, p(x,1-t)\, dxdt .$$
\end{thmappdef}
When it is clear from the context, we will sometimes use $\epsilon_F(f)$ and $\epsilon_{CF}(f)$ for the expected factual and counterfactual losses of an arbitrary function $f: \cX \times \{0,1\} \rightarrow \cY$.

\begin{thmappdef}\label{def:decomploss}
The expected \emph{treated} and \emph{control} losses are:
$$\epsilon^{t=1}_F(h,\Phi) = \int_{\cX } \!\!\! \lyxoneh\, \pt(x)\, dx $$
$$\epsilon^{t=0}_F(h,\Phi) = \int_{\cX } \!\!\! \lyxzeroh\, \pc(x)\, dx $$
$$\epsilon^{t=1}_{CF}(h,\Phi) = \int_{\cX} \!\!\! \lyxoneh\, \pc(x)\, dx $$
$$\epsilon^{t=0}_{CF}(h,\Phi) = \int_{\cX } \!\!\! \lyxzeroh\, \pt(x)\, dx .$$
\end{thmappdef}
The four losses above are simply the loss conditioned on either the control or treated set.  Let $u := p(t=1)$ be the proportion of treated in the population. We then have the immediate result:
\begin{thmapplem}\label{lem:epsdecopm}
$$\epsilon_F(h,\Phi) = u\cdot \epsilon^{t=1}_F(h,\Phi) + (1-u)\cdot \epsilon^{t=0}_F(h,\Phi)$$
$$\epsilon_{CF}(h,\Phi) = (1-u)\cdot \epsilon^{t=1}_{CF}(h,\Phi) + u\cdot  \epsilon^{t=0}_{CF}(h,\Phi).$$
\end{thmapplem}
The proof is immediate, noting that $p(x,t) = u\cdot \pt(x) + (1-u)\cdot \c(x)$, and from the Definitions \ref{def:perunitlossA} and \ref{def:decomploss} of the losses.

\begin{thmappdef}\label{def:ipm}
Let $\cF$ be a function family consisting of functions $g: \cS\rightarrow \R$. For a pair of distributions $p_1$, $p_2$ over $\cS$, define the \emph{Integral Probability Metric}:
$$\text{IPM}_\cF(p_1,p_2) = \sup_{g\in \cF} \left|\int_\cS g(s) \left(p_1(s) - p_2(s)\right) \, ds \right|$$
\end{thmappdef}
$\text{IPM}_\cF(\cdot,\cdot)$ defines a pseudo-metric on the space of probability functions over $\cS$, and for sufficiently large function families, $\text{IPM}_\cF(\cdot,\cdot)$ is a proper metric \cite{muller1997integral}. Examples of sufficiently large functions families includes the set of bounded continuous functions, the set of $1$-Lipschitz functions, and the set of unit norm functions in a universal Reproducing Norm Hilbert Space. The latter two give rise to the Wasserstein and Maximum Mean Discrepancy metrics, respectively \cite{gretton2012mmd,sriperumbudur2012empirical}. We note that for function families $\cF$ such as the three mentioned above, for which $g \in \cF \implies -g \in \cF$, the absolute value can be omitted from definition \ref{def:ipm}.

\subsection{General IPM bound}

We now state and prove the most important technical lemma of this section.
\begin{thmapplem}[Lemma 1, main text]\label{lem:genA}
Let $\Phi: \cX \rightarrow \cR$ be an invertible representation with $\Psi$ its inverse. Let $\pt_\Phi, \pc_\Phi$ be defined as in Definition \ref{def:pphiA}. Let $u = p(t=1)$. Let $\cF$ be a family of functions $g:\cR \rightarrow \R$, and denote by $\text{IPM}_\cF(\cdot, \cdot)$ the integral probability metric induced by $\cF$.  Let $h : \cR \times \{0,1\} \rightarrow \cY$ be an hypothesis. Assume there exists a constant $B_\Phi>0$, such that for $t=0,1$, the function $g_{\Phi,h}(r,t) := \frac{1}{B_\Phi} \cdot  \lythr \in \cF$. Then we have:
\begin{align}
&\epsilon_{CF}(h,\Phi) \leq \nonumber \\
&(1-u)  \epsilon^{t=1}_F(h,\Phi) + u  \epsilon^{t=0}_F(h,\Phi) + \nonumber \\ 
&  B_\Phi \cdot \text{IPM}_\cF\left( \pt_\Phi, \pc_\Phi \right).
\end{align}
\end{thmapplem}

\begin{proof}
\allowdisplaybreaks
\begin{align}
&\epsilon_{CF}(h,\Phi) - \left[ (1-u)\cdot \epsilon^{t=1}_{F}(h,\Phi) + u\cdot  \epsilon^{t=0}_{F}(h,\Phi) \right]= \nonumber \\
&\left[ (1-u)\cdot \epsilon^{t=1}_{CF}(h,\Phi) + u\cdot  \epsilon^{t=0}_{CF}(h,\Phi)\right] - \nonumber \\ 
&\left[ (1-u)\cdot \epsilon^{t=1}_{F}(h,\Phi) + u\cdot  \epsilon^{t=0}_{F}(h,\Phi)\right] = \nonumber  \\
&(1-u)\cdot \left[\epsilon^{t=1}_{CF}(h,\Phi)  - \epsilon^{t=1}_{F}(h,\Phi)\right]  + \nonumber \\
&u\cdot \left[\epsilon^{t=0}_{CF}(h,\Phi)  - \epsilon^{t=0}_{F}(h,\Phi)\right]  = \label{eq:usedef6}\\
&(1-u)\int_{\cX} \!\!\! \lyxoneh\, \left(\pc(x) - \pt(x)\right) \, dx  + \nonumber \\
& u  \int_{\cX} \!\!\!  \lyxzeroh\, \left(\pt(x) - \pc(x) \right)\, dx = \label{eq:usecov} \\
&(1-u)\int_{\cR} \!\!\! \lyonehpsi\, \left(\pc_\Phi(r) - \pt_\Phi(r)\right) \, dr  + \nonumber \\
& u  \int_{\cR} \!\!\!  \lyzerohpsi\, \left(\pt_\Phi(r) - \pc_\Phi(r) \right)\, dr = \nonumber \\
&B_\Phi \cdot (1-u)\int_{\cR} \!\! \frac{1}{B_\Phi}\lyonehpsi\, \left(\pc_\Phi(r) - \pt_\Phi(r)\right) \, dr  + \nonumber \\
& B_\Phi \cdot  u  \int_{\cR} \!\! \frac{1}{B_\Phi} \lyzerohpsi\, \left(\pt_\Phi(r) - \pc_\Phi(r) \right)\, dr \leq \label{proof:ineq} \\
&B_\Phi \cdot (1-u)\sup_{g\in \cF} \left| \int_{\cR} \!\! g(r)\, \left(\pc_\Phi(r) - \pt_\Phi(r)\right) \, dr  \right| + \nonumber \\
& B_\Phi \cdot u \sup_{g\in \cF} \left| \int_{\cR} \!\! g(r)\, \left(\pt_\Phi(r) - \pc_\Phi(r) \right)\, dr \right| = \label{eq:useipmdef}\\
&B_\Phi \cdot \text{IPM}_\cF(\pc_\Phi,\pt_\Phi) .
\end{align}
Equality \eqref{eq:usedef6} is by Definition \ref{def:decomploss} of the treated and control loss, equality \eqref{eq:usecov} is by the change of variables formula and Definition \ref{def:pphiA} of $\pt_\Phi$ and $\pc_\Phi$, inequality \eqref{proof:ineq} is by the premise that $\frac{1}{B_\Phi} \cdot  \lythr \in \cF$ for $t=0,1$, and \eqref{eq:useipmdef} is by Definition \ref{def:ipm} of an IPM.
\end{proof}

The essential point in the proof of Lemma \ref{lem:genA} is inequality \ref{proof:ineq}. Note that on the l.h.s. of the inequality, we need to evaluate the expectations of $ \lyzerohpsi$ over $\pt_\Phi$ and $\lyonehpsi$ over $\pc_\Phi$. Both of these expectations are in general unavailable, since they require us to evaluate treatment outcomes on the control, and control outcomes on the treated. We therefore upper bound these unknowable quantities by taking a supremum over a function family which includes $\lyzerohpsi$ and $\lyonehpsi$. The upper bound ignores most of the details of the outcome, and amounts to measuring a distance between two distributions we have samples from: the control and treated distribution. Note that for a randomized trial (i.e. when $t \indep x$) with we have that $\text{IPM}(\pt_\Phi,\pc_\Phi) = 0$. Indeed, it is straightforward to show that in that case we actually have an equality: $\epsilon_{CF}(h,\Phi) = (1-u)\cdot \epsilon^{t=1}_{F}(h,\Phi) + u\cdot  \epsilon^{t=0}_{F}(h,\Phi)$.

The crucial condition in Lemma \ref{lem:genA} is that the function $g_{\Phi,h}(r) := \frac{1}{B_\Phi} \lythr$ is in $\cF$. In subsections \ref{subsec:wass} and \ref{subsec:mmd} below we look into two specific function families $\cF$, and evaluate what does this inclusion condition entail, and in particular we will derive specific bounds for $B_\Phi$.

\begin{thmappdef}\label{def:m}
For $t=0,1$ define:
$$m_t(x) := \E\left[Y_t|x\right].$$
\end{thmappdef}
Obviously for the treatment effect $\tau(x)$ we have $\tau(x) = m_1(x)-m_0(x)$.

Let $f: \cX \times \{0,1\} \rightarrow \cY$ by an hypothesis, such that $f(x,t) = h(\Phi(x),t)$ for a representation $\Phi$ and hypothesis $h$ defined over the output of $\Phi$.
\begin{thmappdef}\label{def:inditeerrA}
The treatment effect estimate for unit $x$ is:
\begin{align*}
\hat{\tau}_f  (x) = f(x,1) - f(x,0).
\end{align*}
\end{thmappdef}

\begin{thmappdef}
 The expected Precision in Estimation of Heterogeneous Effect (PEHE) loss of $g$ is:
$$\epehe(f) = \int_{\cX } \left( \hat{\tau}_f(x) - \tau(x) \right)^2 \, p(x) \, dx .$$
\end{thmappdef}

\begin{thmappdef}\label{def:sigma2}
The expected variance of $Y_t$ with respect to a distribution $p(x,t)$:
$$\sigma^2_{Y_t}(p(x,t)) = \int_{\cX \times \cY} \left(Y_t - m_t(x)\right)^2 p(Y_t|x)p(x,t) \, dY_t dx .$$
We define:
\begin{align*}
&\sigma^2_{Y_t} = \min\{\sigma^2_{Y_t}(p(x,t)), \sigma^2_{Y_t}(p(x,1-t))\} , \\
&\sigma^2_{Y} = \min\{\sigma^2_{Y_0},\sigma^2_{Y_1}\}.
\end{align*}
\end{thmappdef}
If $Y_t$ are deterministic functions of $x$, then $\sigma^2_{Y} = 0$.

We now show that $\epehe(f)$ is upper bounded by $2 \epsilon_F + 2 \epsilon_{CF} - 2\sigma^2_Y$ where $\epsilon_F$ and $\epsilon_{CF}$ are w.r.t. to the squared loss. An analogous result can be obtained for the absolute loss, using mean absolute deviation. 

\begin{thmapplem}\label{lem:epsmloss}
For any function $f : \cX \times \{0,1\} \rightarrow \cY$, and distribution $p(x,t)$ over $\cX \times \{0,1\}$:
\begin{align*}
\int_\cX& \left(f(x,t) - m_t(x) \right)^2 \, p(x,t) \, dx dt = \\
&\epsilon_F(f) - \sigma^2_{Y_t}(p(x,t)), \\
 \int_\cX &\left(f(x,t) - m_t(x) \right)^2 \, p(x,1-t) \, dx dt = \\
& \epsilon_{CF}(f) - \sigma^2_{Y_t}(p(x,1-t)),
\end{align*}
where $\epsilon_F(f)$ and $\epsilon_{CF}(f)$ are w.r.t. to the squared loss.
\end{thmapplem}
\begin{proof}
For simplicity we will prove for $p(x,t)$ and $\epsilon_F(f)$. The proof for $p(x,1-t)$ and $\epsilon_{CF}$ is identical.
\begin{align}
&\epsilon_F(f) = \nonumber \\
&\int_{\cX  \times \{0,1\} \times \cY}   \!\!\!\!\!\!\!\!\!\!\!\!\!\!\!\!\!\!\!\!\!\! \left(f(x,t) - Y_t\right)^2 p(Y_t|x) p(x,t) \, dY_t dx dt = \nonumber \\
&\int_{\cX  \times \{0,1\} \times \cY}   \!\!\!\!\!\!\!\!\!\!\!\!\!\!\!\!\!\!\!\!\!\! \left(f(x,t) - m_t(x)\right)^2 p(Y_t|x) p(x,t) \, dY_t dx dt + \nonumber \\
&\int_{\cX  \times \{0,1\} \times \cY}   \!\!\!\!\!\!\!\!\!\!\!\!\!\!\!\!\!\!\!\!\!\!  \left(m_t(x) - Y_t\right)^2 p(Y_t|x) p(x,t) \, dY_t dx dt + \label{eq:expctzero} \\
&\int_{\cX  \times \{0,1\} \times \cY}  \!\!\!\!\!\!\!\!\!\!\!\!\!\!\!\!\!\!\!\!\!\! \left(f(x,t) - m_t(x)\right)\left(m_t(x) - Y_t\right) p(Y_t|x) p(x,t) \, dY_t dx dt = \label{eq:epssigm} \\
&\int_{\cX  \times \{0,1\} }   \!\!\!\!\!\!\!\!\!\!\!\!\!\!\!\!\!\! \left(f(x,t) - m_t(x)\right)^2  p(x,t) \,  dx dt + \nonumber \\
&\quad \quad \sigma^2_{Y_0}(p(x,t)) + \sigma^2_{Y_1}(p(x,t)) +0 , \nonumber
\end{align}
where the equality \eqref{eq:epssigm} is by the Definition \ref{def:sigma2} of $\sigma^2_{Y_t}(p)$, and because the integral in \eqref{eq:expctzero} evaluates to zero, since $m_t(x) = \int_\cX Y_t p(Y_t|x) \, dx$.
\end{proof}

\begin{thmappthm}\label{thm:indtausqloss}
Let $\Phi : \cX \rightarrow \cR$ be a one-to-one representation function, with inverse $\Psi$. Let $\pt_\Phi, \pc_\Phi$ be defined as in Definition \ref{def:pphiA}. Let $u = p(t=1)$. Let $\cF$ be a family of functions $g:\cR \rightarrow \R$, and denote by $\text{IPM}_\cF(\cdot, \cdot)$ the integral probability metric induced by $\cF$.  Let $h : \cR \times \{0,1\} \rightarrow \cY$ be an hypothesis. Let the loss $L(y_1,y_2) = (y_1 - y_2)^2$.  Assume there exists a constant $B_\Phi>0$, such that for $t \in \{0,1\}$, the functions $g_{\Phi,h}(r,t) := \frac{1}{B_\Phi} \cdot  \lythr \in \cF$.
We then have:
\begin{align*}
&\epehe(h,\Phi) \leq  \nonumber \\
& 2\!\left(\epsilon_{CF}(h,\Phi) + \epsilon_F(h,\Phi) - 2\sigma^2_Y \right) \leq  \nonumber \\
&2\!\left(\epsilon_F^{t=0}(h,\Phi)\! +\!\epsilon_F^{t=1}(h,\Phi)\!+ \! B_\Phi  \text{IPM}_\cF \left( \pt_\Phi, \pc_\Phi \right)\! -\! 2\sigma^2_Y\right)\!,
\end{align*}
where $\epsilon_F$ and $\epsilon_{CF}$ are with respect to the squared loss.
\end{thmappthm}

\begin{proof}
We will prove the first inequality, $\epehe(f) \leq   2\epsilon_{CF}(h,\Phi) + 2\epsilon_F(h,\Phi) - 2\sigma^2_Y$. The second inequality is then immediate by Lemma \ref{lem:genA}.
Recall that we denote $\epehe(f) = \epehe(h,\Phi)$ for $f(x,t) = h(\Phi(x),t)$.
%We will bound $\int_{\cX } \left| \hat{\tau}_f(x) - \tau(x) \right| \, p(x) \, dx$.
\begin{align}
&\epehe(f) = \nonumber\\
& \int_\cX \bigl(\left(f(x,1) - f(x,0)\right) - \left(m_1(x) - m_0(x)\right) \bigr)^2 p(x) \, dx= \nonumber\\
& \int_\cX \bigl( \left(f(x,1) - m_1(x)\right) + \left(m_0(x) - f(x,0)\right) \bigr)^2 p(x) \, dx \leq  \label{eq:sqineq}\\
&  2\int_\cX \left(\left(  f(x,1) - m_1(x) \right)^2  + \left(m_0(x) - f(x,0) \right)^2\right) p(x) \, dx = \label{eq:use:margt2sq} \\
& 2\int_\cX  \left(f(x,1) - m_1(x)\right)^2  p(x,t=1)  \,dx+ \nonumber \\
& \quad 2\int_\cX \left(m_0(x) - f(x,0) \right)^2  p(x,t=0)  \, dx +  \nonumber\\
& \quad 2\int_\cX  \left(f(x,1) - m_1(x)\right)^2 p(x,t=0)  \, dx+ \nonumber\\
& \quad 2\int_\cX \left(m_0(x) - f(x,0) \right)^2  p(x,t=1)  \, dx \nonumber = \\
&2 \int_\cX \left(f(x,t) - m_t(x)\right)^2 p(x,t) \, dx dt + \nonumber \\
&\quad 2 \int_\cX \left(f(x,t) - m_t(x)\right)^2 p(x,1-t)\, dx dt \leq \label{eq:usesigmineq}\\
&2 (\epsilon_F - \sigma^2_Y) + 2 (\epsilon_{CF} - \sigma^2_Y). \nonumber
\end{align}
where \eqref{eq:sqineq} is because $(x+y)^2 \leq 2(x^2 + y^2)$, \eqref{eq:use:margt2sq} is because $p(x) = p(x,t=0) + p(x,t=1)$ and \eqref{eq:usesigmineq} is by Lemma \ref{lem:epsmloss} and Definition \ref{def:epfepcf} of the losses $\epsilon_F$, $\epsilon_{CF}$ and Definition \ref{def:sigma2} of $\sigma^2_Y$.
Having established the first inequality in the Theorem statement, we now show the second.
We have by Lemma \ref{lem:genA} that:
\begin{align*}
&\epsilon_{CF}(h,\Phi) \leq  \\ 
&(1-u)  \epsilon^{t=1}_F(h,\Phi) + u  \epsilon^{t=0}_F(h,\Phi) +   B_\Phi \cdot \text{IPM}_\cF\left( \pt_\Phi, \pc_\Phi \right).
\end{align*}
We further have by Lemma \ref{lem:epsdecopm} that:
\begin{align*} 
\epsilon_F(h,\Phi) = u \epsilon^{t=1}_F(h,\Phi) + (1-u) \epsilon^{t=0}_F(h,\Phi).
\end{align*}
Therefore 
\begin{align*}
&\epsilon_{CF}(h,\Phi)+\epsilon_F(h,\Phi) \leq  \\
&\epsilon^{t=1}_F(h,\Phi) + \epsilon^{t=0}_F(h,\Phi)+ B_\Phi  \text{IPM}_\cF\left( \pt_\Phi, \pc_\Phi \right).
\end{align*}
\end{proof}

The upper bound is in terms of the standard generalization error on the treated and control distributions separately. Note that in some cases we might have very different sample sizes for treated and control, and that will show up in the finite sample bounds of these generalization errors.

We also note that the upper bound can be easily adapted to the case of the absolute loss PEHE $\left| \hat{\tau}(x) - \tau(x)\right|$. In that case the upper bound in the Theorem will have a factor $1$ instead of the $2$ stated above, and the standard deviation $\sigma^2_Y$ replaced by mean absolute deviation. The proof is straightforward where one simply applies the triangle inequality in inequality \eqref{eq:sqineq}.

We will now give specific upper bounds for the constant $B_\Phi$ in Theorem \ref{thm:indtausqloss}, using two function families $\cF$ in the IPM: the family of $1$-Lipschitz functions, and the family of $1$-norm reproducing kernel Hilbert space functions. Each one will have different assumptions about the distribution $p(x,t,Y_0,Y_1)$ and about the representation $\Phi$ and hypothesis $h$.

\subsection{The family of $1$-Lipschitz functions}\label{subsec:wass}

For $\cS \subset \R^d$, a function $f : \cS \rightarrow \R$ has Lipschitz constant $K$ if for all $x,y \in \cS$, $|f(x) - f(y)| \leq K\|x-y\|$. If $f$ is differentiable, then a sufficient condition for $K$-Lipschitz constant is if $\| \frac{\partial f}{\partial s} \| \leq K$ for all $s\in \cS$.

For simplicity's sake we assume throughout this subsection that the true labeling functions the densities $p(Y_t|x)$ and the loss $L$ are differentiable. However, this assumption could be relaxed to a mere Lipschitzness assumption.
\begin{thmappasmp}\label{assmp_f_lip}
There exists a constant $K >0$ such that for all $x \in \cX$, $t \in \{0,1\}$, $\| \frac{p(Y_t|x)}{\partial x} \| \leq K$.
\end{thmappasmp}
Assumption \ref{assmp_f_lip} entails that each of the potential outcomes change smoothly as a function of the covariates (context) $x$.

\begin{thmappasmp}\label{asmp:loss}
The loss function $L$ is differentiable, and there exists a constant $K_L >0$ such that $\left|\frac{d L(y_1,y_2)}{d y_i}\right| \leq K_L$ for $i=1,2$. Additionally, there exists a constant $M$ such that for all $y_2\in \cY$, $M \geq \int_{\cY} L(y_1,y_2)\, dy_1$.
\end{thmappasmp}
Assuming $\cY$ is compact, loss functions which obey Assumption \ref{asmp:loss} include the log-loss, hinge-loss, absolute loss, and the squared loss.

When we let $\cF$ in Definition \ref{def:ipm} be the family of $1$-Lipschitz functions, we obtain the so-called \emph{$1$-Wasserstein} distance between distributions, which we denote $\text{Wass}(\cdot,\cdot)$.
It is well known that $\text{Wass}(\cdot,\cdot)$ is indeed a metric between distributions \cite{villani2008optimal}.

\begin{thmappdef}\label{def:rcondA}
Let $\Jphix$ be the Jacobian matrix of $\Phi$ at point $x$, i.e. the matrix of the partial derivatives of $\Phi$. Let $\sigma_{max}(A)$ and $\sigma_{min}(A)$ denote respectively the largest and smallest singular values of a matrix $A$.
Define $\rho(\Phi) = \sup_{x \in \cX} \, \, \sigma_{max}\left(\Jphix\right) / \sigma_{min}\left(\Jphix\right)$.
\end{thmappdef}
It is an immediate result that $\rho(\Phi) \geq 1$.

\begin{thmappdef}
We will call a representation function $\Phi : \cX \rightarrow \cR$ \emph{Jacobian-normalized} if $\sup_{x \in \cX} \sigma_{max} \left(\Jphix \right) = 1$.
\end{thmappdef}
Note that any non-constant representation function $\Phi$ can be Jacobian-normalized by a simple scalar multiplication.

\begin{thmapplem}\label{lem:lip}
Assume that $\Phi$ is a Jacobian-normalized representation, and let $\Psi$ be its inverse.
For $t=0,1$, the Lipschitz constant of $p(Y_t|\Psi(r))$ is bounded by $\rho(\Phi) K$, where  $K$ is from Assumption \ref{assmp_f_lip}, and $\rho(\Phi)$ as in Definition \ref{def:rcondA}.
\end{thmapplem}
\begin{proof}
Let $\Psi: \cR \rightarrow \cX$ be the inverse of $\Phi$, which exists by the assumption that $\Phi$ is one-to-one. Let $\dPhi$ be the Jacobian matrix of $\Phi$ evaluated at $x$, and similarly let $\dPsi$ be the Jacobian matrix of $\Psi$ evaluated at $r$. Note that $\dPsi \cdot \dPhi = I$ for $r= \Phi(x)$, since $\Psi \circ \Phi$ is the identity function on $\cX$.	Therefore for any $r \in \cR$ and $x= \Psi(r)$:
\begin{equation}\label{eqn:sigman2}
\sigma_{max}\left(\dPsi\right) = \frac{1}{\sigma_{min}\left(\dPhi\right)},
\end{equation}
where $\sigma_{max}(A)$ and $\sigma_{min}(A)$ are respectively the largest and smallest singular values of the matrix $A$, i.e. $\sigma_{max}(A)$ is the spectral norm of $A$.

For $x = \Psi(r)$ and $t \in \{0,1\}$, we have by the chain rule:
\begin{align}
& \| \frac{\partial p(Y_t|\Psi(r))}{\partial r} \| =  \| \frac{\partial p(Y_t|\Psi(r))}{\partial \Psi(r)}  \dPsi\|  \leq \label{eqn:matnorm}\\
& \| \dPsi\|  \| \frac{\partial p(Y_t|\Psi(r))}{\partial \Psi(r)} \|  = \label{eqn:invjac} \\
& \frac{1}{\sigma_{min}\left(\dPhi\right)} \| \frac{\partial p(Y_t|x)}{\partial x} \| \leq \label{eqn:maxgrad}\\
&  \frac{K}{\sigma_{min}\left(\dPhi\right)} \leq  \rho(\Phi) K \label{eqn:ineq_vphi},
\end{align}
where inequality \eqref{eqn:matnorm} is by the matrix norm inequality, equality \eqref{eqn:invjac} is by \eqref{eqn:sigman2}, inequality \eqref{eqn:maxgrad} is by assumption \ref{assmp_f_lip} on the norms of the gradient of $p(Y_t|x)$ w.r.t $x$ , and inequality \eqref{eqn:ineq_vphi} is by Definition \ref{def:rcondA} of $\rho(\Phi)$, the assumption that $\Phi$ is Jacobian-normalized, and noting that singular values are necessarily non-negative.

%Inequalities \eqref{eqn:matnorm} - \eqref{eqn:ineq_vphi}, along with Assumption \ref{assmp1} and \ref{lemma:lip_sum}, show us that the Lipschitz constant of %$\tilde{f_t}$ is bounded by $\sqrt{\frac{K^2}{m(\Phi)^2} + b^2 K^2} = K \sqrt{\frac{1}{m(\Phi)^2} + b^2}$.

\end{proof}

\begin{thmapplem}\label{lem:errlip}
Under the conditions of Lemma \ref{lem:genA}, further assume that for $t=0,1$, $p(Y_t|x)$ has gradients bounded by $K$ as in \ref{assmp_f_lip}, that $h$ has bounded gradient norm $b K$, that the loss $L$ has bounded gradient norm $K_L$, and that $\Phi$ is Jacobian-normalized. Then the Lipschitz constant of $ \lythr$ is upper bounded by $K_L \cdot K \left(M \rho(\Phi) +  b \right)$ for $t=0,1$.
\end{thmapplem}
\begin{proof}
Using the chain rule, we have that:
\begin{align}
& \|\frac{\partial \lythr}{\partial r}\| =\| \frac{\partial}{\partial r} \int_\cY L(Y_t,h(r,t))  p(Y_t|r) dY_t \|= \nonumber \\
& \| \int_\cY \frac{\partial}{\partial r} \left[ L(Y_t,h(r,t))  p(Y_t|r)\right] \, dY_t \| = \nonumber \\
& \| \int_\cY \! p(Y_t|r)\frac{\partial}{\partial r} \! L(Y_t,h(r,t)) \!   +\! L(Y_t,h(r,t)) \frac{\partial}{\partial r} \! p(Y_t|r)  dY_t \|  \leq  \nonumber \\
& \int_\cY p(Y_t|r) \| \frac{\partial}{\partial r} L(Y_t,h(r,t)) \| \, dY_t + \nonumber\\
& \int_\cY  L(Y_t,h(r,t)) \frac{\partial}{\partial r} p(Y_t|r) \, dY_t \leq \label{lip_const_ineq1}\\
& \int_\cY p(Y_t|r) \| \frac{\partial L(Y_t,h(r,t))}{\partial h(r,t) }  \frac{\partial h(r,t)}{\partial r} \| \, dY_t + \nonumber\\
& \int_\cY  L(Y_t,h(r,t)) \frac{\partial}{\partial r} p(Y_t|r) \, dY_t \leq \label{lip_const_ineq2}\\
& \int_\cY p(Y_t|r) K_L \cdot b \cdot K + M \cdot \rho(\Phi) \cdot  K,
\end{align}
%\begin{align*}
%&\frac{\partial L}{\partial r} = \frac{\partial L}{\partial Y_t(\Psi(r))} \frac{\partial Y_t(\Psi(r))}{\partial r} + \frac{\partial L}{\partial h(r,t)} \frac{\partial h(r,t)}{\partial r} \leq \\
%&  K_l K \rho(\Phi) + K_l \cdot  b K = K_L \cdot K \left(\rho(\Phi) +  b \right),
%\end{align*}
where inequality \ref{lip_const_ineq1} is due to Assumption \ref{asmp:loss} and inequality \ref{lip_const_ineq2} is due to Lemma \ref{lem:lip}.
\end{proof}

\begin{thmapplem}\label{thm:wasA}
Let $u = p(t=1)$ be the marginal probability of treatment, and assume $0<u<1$. Let $\Phi : \cX \rightarrow \cR$ be a one-to-one, Jacobian-normalized representation function. Let $K$ be the Lipschitz constant of the functions $p(Y_t|x)$ on $\cX$. Let $K_L$ be the Lipschitz constant of the loss function $L$, and $M$ be as in Assumption \ref{asmp:loss}. Let $h : \cR \times \{0,1\} \rightarrow \R$ be an hypothesis with Lipschitz constant $b K$.
%Denote by $\text{Wass}(u \cdot \pt_\Phi,(1-u) \cdot \pc_\Phi)$ the Wasserstein distance between the treatment and control distributions induced by $\Phi$.
Then:
\begin{align}\label{eq:thm_wasA}
&\epsilon_{CF}(h,\Phi) \leq  \nonumber \\
&(1-u)  \epsilon^{t=1}_F(h,\Phi) + u  \epsilon^{t=0}_F(h,\Phi)  + \nonumber \\
& 2 \left(M \rho(\Phi) +b\right) \cdot K  \cdot K_L \cdot \text{Wass}(\pt_\Phi \, , \pc_\Phi).
\end{align}
\end{thmapplem}

\begin{proof}
We will apply Lemma \ref{lem:genA} with $\cF = \{g: \cR \rightarrow \R \text{ s.t. } f  \text{ is } 1 \text{-Lipschitz}\}$. By Lemma \ref{lem:errlip}, we have that for $B_\Phi = \left(M \rho(\Phi) +b\right) \cdot K  \cdot K_L$, the function $\frac{1}{B_\Phi} \lythr \in \cF$. Inequality \eqref{eq:thm_wasA} then holds as a special case of Lemma \ref{lem:genA}.
\end{proof}
\begin{thmappthm}%[Theorem 2, main paper]
Under the assumptions of Lemma \ref{thm:wasA}, using the squared loss for $\epsilon_F$, we have:
\begin{align*}
&\epehe(h,\Phi) \leq \nonumber \\
&2 \epsilon^{t=0}_F(h,\Phi)  + 2\epsilon^{t=1}_F(h,\Phi)  - 4\sigma^2_Y +  \nonumber \\
&2 \left(M \rho(\Phi) +b\right) \cdot K  \cdot K_L \cdot \text{Wass}( \pt_\Phi \, ,  \pc_\Phi) .
\end{align*}
\end{thmappthm}
\begin{proof}
Plug in the upper bound of Lemma \ref{thm:wasA} into the upper bound of Theorem \ref{thm:indtausqloss}.
\end{proof}

We examine the constant $\left(M \rho(\Phi) +b\right) \cdot K \cdot K_L$ in Theorem \ref{thm:wasA}. $K$, the Lipschitz constant of $m_0$ and $m_1$, is not under our control and measures an aspect of the complexity of the true underlying functions we wish to approximate. The terms $K_L$ and $M$ depend on our choice of loss function and the size of the space $\cY$. The term $b$ comes from our assumption that the hypothesis $h$ has norm $ b K$. Note that smaller $b$, while reducing the bound, might force the factual loss term $\epsilon_F(h,\Phi)$ to be larger since a small $b$ implies a less flexible $h$. Finally, consider the term $\rho(\Phi)$.
The assumption that $\Phi$ is normalized is rather natural, as we do not expect a certain scale from a representation. Furthermore, below we show that in fact the Wasserstein distance is positively homogeneous with respect to the representation $\Phi$. Therefore, in Lemma \ref{thm:wasA}, we can indeed assume that $\Phi$ is normalized. The specific choice of \emph{Jacobian-normalized} scaling yields what is in our opinion a more interpretable result in terms of the inverse condition number $\rho(\Phi)$. For twice-differentiable $\Phi$, $\rho(\Phi)$ is minimized if and only if $\Phi$ is a linear orthogonal transformation \cite{mathover}. %\todo[inline]{Fill in details for upper bounding $\rho(\Phi)$.}

\begin{thmapplem}\label{lem:wassinvar}
The Wasserstein distance is positive homogeneous for scalar transformations of the underlying space. Let $p$, $q$ be probability density functions defined over $\cX$. For $\alpha>0$ and
 the mapping $\Phi(x) = \alpha x$, let $p_\alpha$ and $q_\alpha$ be the distributions on $\alpha\cX$ induced by $\Phi$. Then:
$$\text{Wass}\left(p_{\alpha  },q_{\alpha  }\right) = \alpha\text{Wass}\left(p,q\right).$$
\end{thmapplem}

\begin{proof}
Following \cite{villani2008optimal,tabak2016preconditioning}, we use another characterization of the Wasserstein  distance. Let $\mathcal{M}_{p,q}$ be the set of mass preserving maps from $\cX$ to itself which map the distribution $p$ to the distribution $q$. That is,  $\mathcal{M}_{p,q} = \{M : \cX \rightarrow \cX \text{ s.t. }q(M(S)) = p(S) \text{ for all measurable bounded } S \subset \cX\}$. We then have that:
\begin{equation}\label{eq:wass_def_M}
\text{Wass}(p,q) = \inf_{M \in \mathcal{M}_{p,q}} \int_\cX \|M(x) - x \| p(x) \, dx .
\end{equation}
It is known that the infimum in \eqref{eq:wass_def_M} is actually achievable \citep[Theorem 5.2]{villani2008optimal}. Denote by $M^*: \cX \rightarrow \cX$ the map achieving the infimum for $\text{Wass}(p,q)$ . Define $M^*_\alpha :\alpha \cX \rightarrow \alpha \cX$, by $M^*_\alpha(x') = \alpha M^*(\frac{x'}{\alpha})$, where $x' = \alpha x$.  $M^*_\alpha$ maps $p_\alpha$ to $q_\alpha$, and we have that $\|M^*_\alpha(x') - x'\| = \alpha \|M^*(x)-x\|$. Therefore $M^*_\alpha$ achieves the infimum for the pair $(p_\alpha,q_\alpha)$, and we have that $\text{Wass}\left(p_{\alpha  },q_{\alpha  }\right) = \alpha\text{Wass}\left(p,q\right)$.
\end{proof}

\subsection{Functions in the unit ball of a RKHS}\label{subsec:mmd}

Let $\cH_x, \cH_r$ be a reproducing kernel Hilbert space, with corresponding kernels $k_x(\cdot,\cdot)$, $k_r(\cdot,\cdot)$. We have for all $x \in \cX$ that  $k_x(\cdot,x)$ is its Hilbert space mapping, and similarly $k_r(\cdot,r)$ for all $r \in \cR$.

Recall that the major condition in Lemma \ref{lem:genA} is that $\frac{1}{B_\Phi} \lythr \in \cF$. The function space $\cF$ we use here is $\cF = \{ g\in \cH_r \text{ s.t. } \|g\|_{\cH_r} \leq 1\}$.

We will focus on the case where $L$ is the squared loss, and we will make the following two assumptions:

\begin{thmappasmp}\label{asmp:yinrkhs}
There exist $f^Y_0,f^Y_1 \in \cH_x$ such that $m_t(x) = \left<f^Y_t,k_x(x,\cdot)\right>_{\cH_x}$, i.e.
the mean potential outcome functions $m_0,m_1$ are in $\cH_x$.   Further assume that $\|f^Y_t\|_{\cH_x} \leq K$.
\end{thmappasmp}

\begin{thmappdef}\label{def:eta}
Define $\eta_{Y_t}(x) :=\sqrt{\int_\cY \left(Y_t - m_t(x)\right)^2 p(Y_t|x)}$. $\eta_{Y_t}(x)$ is the standard deviation of $Y_t|x$.
\end{thmappdef}

\begin{thmappasmp}\label{asmp:etainrkhs}
There exists $f^\eta_0, f^\eta_1 \in \cH_x$ such that $\eta_{Y_t}(x) = \left<f^\eta_t,k_x(x,\cdot)\right>_{\cH_x}$, i.e. the conditional standard deviation functions of $Y_t|x$ are in $\cH_x$. Further assume that $\|f^\eta_t\|_{\cH_x} \leq M$.
\end{thmappasmp}
\begin{thmappasmp}\label{asmp:phiinrkhs}
Let $\Phi: \cX \rightarrow \cY$ be an invertible representation function, and let $\Psi$ be its inverse.
We assume there exists a bounded linear operator $\GP: \cH_r \rightarrow \cH_x$ such that $\left<f^Y_t,k_x(\Psi(r),\cdot)\right>_{\cH_x} = \left<f^Y_t,\GP k_r(r,\cdot)\right>_{\cH_x}$. We further assume that the Hilbert-Schmidt norm (operator norm) $\|\GP\|_{HS}$ of $\GP$ is bounded by $K_\Phi$.
\end{thmappasmp}

The two assumptions above amount to assuming that $\Phi$ can be represented as one-to-one linear map between the two Hilbert spaces $\cH_x$ and $\cH_r$.

Under Assumptions \ref{asmp:yinrkhs} and \ref{asmp:phiinrkhs} about $m_0,m_1$, and $\Phi$, we have that $m_t(\Psi(r)) = \left<\GP^* f^Y_t, k_r(r,\cdot)\right>_{\cH_r}$, where $\GP^*$ is the adjoint operator of $\GP$ \cite{grunewalder2013smooth}.

\begin{thmapplem}\label{lem:mmdbound}
Let $h : \cR \times \{0,1\} \rightarrow \R$ be an hypothesis, and assume that there exist $f^h_t \in \cH_r$ such that $h(r,t) = \left<f^h_t,k_r(r,\cdot)\right>_{\cH_r}$, and such that $\|f^h_t\|_{\cH_r} \leq b$.
Under Assumption \ref{asmp:yinrkhs} about $m_0,m_1$, we have that $\lythr = \int_\cY \left(Y_t - h(r,t)\right)^2 p(Y_t|r) dY_t$ is in the tensor Hilbert space $\cH_r \otimes \cH_r$. Moreover, the norm of $\lythr$ in $\cH_r \otimes \cH_r$ is upper bounded by $4\left(K_\Phi^2 K^2 + b^2\right)$.
\end{thmapplem}
\begin{proof}
We first decompose $\int_\cY \left(Y_t - h(r,t)\right)^2 p(Y_t|x) dY_t$ into a noise and mean fitting term, using $r=\Phi(x)$:
\begin{align}
&\lythr = \nonumber \\
&\int_\cY \left(Y_t - h(r,t)\right)^2 p(Y_t|r) \, dY_t = \nonumber \\
&\int_\cY \left(Y_t - m_t(x) + m_t(x) - h(\Phi(x),t)\right)^2 p(Y_t|x) \, dY_t =\nonumber \\
&\int_\cY \left(Y_t - m_t(x) \right)^2 p(Y_t|x) \, dY_t + \nonumber \\
& \quad  \left(m_t(x) - h(\Phi(x),t)\right)^2  + \nonumber \\
& \quad 2 \int_\cY \left(Y_t - m_t(x) \right)\left(m_t(x) - h(\Phi(x),t)\right) p(Y_t|x) dY_t   =\label{eq:expctzero2} \\
& \eta^2_{Y_t}(x) +  \left(m_t(x) - h(\Phi(x),t)\right)^2  + 0, \label{eq:mmddecomp}
\end{align}
where equality \eqref{eq:expctzero2} is by Definition \ref{def:eta} of $\eta$, and because $\int_\cY \left(Y_t - m_t(x)\right) p(Y_t|x) \, dY_t = 0$ by definition of $m_t(x)$.

Moving to $\cR$, recall that $r=\Phi(x)$, $x = \Psi(r)$.
By linearity of the Hilbert space, we have that $m_t(\Psi(r)) - h(r,t) = \left<\GP^* f^Y_t, k_r(r,\cdot)\right>_{\cH_r} - \left<f^h_t,k_r(r,\cdot)\right>_{\cH_r} = \left<\GP^* f^Y_t - f^h_t,k_r(r,\cdot)\right>_{\cH_r}$.
By a well known result \citep[Theorem 7.25]{steinwart2008support}, the product $(Y_t(\Psi(r)) - h(r,t) ) \cdot (Y_t(\Psi(r)) - h(r,t) )$ lies in the tensor product space  $\cH_r \otimes \cH_r$, and is equal to $\left< (\GP^* f^Y_t - f^h_t) \otimes (\GP^* f^Y_t - f^h_t) , k_r(r,\cdot) \otimes k_r(r,\cdot) \right>_{\cH_r \otimes \cH_r}$. The norm of this function in $\cH_r \otimes \cH_r$ is $\| \GP^* f^Y_t - f^h_t \|^2_{\cH_r}$. This is the general Hilbert space version of the fact that for a vector $w \in \R^d$ one has that $\|w w^\top\|_F = \|w\|_2^2$, where $\|\cdot\|_F$ is the matrix Frobenius norm, and $\| \cdot\|^2_2$ is the square of the standard Euclidean norm.
We therefore have a similar result for $\eta^2_{Y_t}$, using Assumption \ref{asmp:etainrkhs}: $\eta^2_{Y_t}(x) = \eta^2_{Y_t}(\Psi(r)) = \left<\GP^* f^\eta_t \otimes \GP^* f^\eta_t, k_r(r,\cdot) \otimes k_r(r,\cdot) \right>_{\cH_r \otimes \cH_r}$.  
The norm of this function in $\cH_r \otimes \cH_r$ is $\| \GP^* f^\eta_t \|^2_{\cH_r}$. 
Overall this leads us to conclude, using Equation \eqref{eq:mmddecomp} that $\lythr \in  \cH_r \otimes \cH_r$.
Now we have, using \eqref{eq:mmddecomp}:
\begin{align}
&\|\lythr\|_{\cH_r \otimes \cH_r} = \nonumber \\
&\|(\GP^* f^Y_t - f^h_t) \otimes (\GP^* f^Y_t - f^h_t)  + \GP^* f^\eta_t \otimes \GP^* f^\eta_t\|_{\cH_r \otimes \cH_r} \leq \label{eq:triineqmmd} \\
&\| \GP^* f^Y_t - f^h_t \|^2_{\cH_r} + \| \GP^* f^\eta_t \|^2_{\cH_r}  \label{eq:mmdineq1} \leq \\
&2 \|\GP^* f^Y_t\|^2_{\cH_r} + 2\|f^h_t \|^2_{\cH_r}  + \| \GP^* f^\eta_t \|^2_{\cH_r} \label{eq:mmdineq2} \leq \\
&  \|\GP^*\|^2_{HS} \left(2\|f^Y_t\|^2_{\cH_x} + \|f^\eta_t\|^2 _{\cH_x} \right)+ 2\|f^h_t \|^2_{\cH_r} =  \label{eq:mmdeq1}\\
&  \|\GP\|^2_{HS} \left(2\|f^Y_t\|^2_{\cH_x} + \|f^\eta_t\|^2 _{\cH_x} \right)+ 2\|f^h_t \|^2_{\cH_r} \leq \label{eq:mmdineq3}\\
&2 K_\Phi^2 (K^2 + M^2) + 2 b^2. \nonumber
\end{align}
Inequality \eqref{eq:triineqmmd} is by the norms given above and the triangle inequality.
Inequality \eqref{eq:mmdineq1} is because for any Hilbert space $\cH$, $\|a-b\|^2_\cH \leq 2 \|a\|_\cH^2 + 2 \|b\|_\cH^2$. Inequality \eqref{eq:mmdineq2} is by the definition of the operator norm. Equality \eqref{eq:mmdeq1} is because the norm of the adjoint operator is equal to the norm of the original operator, where we abused the notation $\|\cdot \|_{HS}$ to mean both the norm of operators from $\cH_x$ to $\cH_r$ and vice-versa. Finally, inequality \eqref{eq:mmdineq3} is by Assumptions \ref{asmp:yinrkhs}, \ref{asmp:etainrkhs} and \ref{asmp:phiinrkhs}, and by the Lemma's premise on the norm of $f^h_T$.
\end{proof}

\begin{thmapplem}\label{thm:mmdA}
Let $u = p(t=1)$ be the marginal probability of treatment, and assume $0<u<1$. Assume the distribution of $Y_t$ conditioned on $x$ follows Assumptions \ref{asmp:etainrkhs} with constant $M$.
Let $\Phi : \cX \rightarrow \cR$ be a one-to-one representation function which obeys Assumption \ref{asmp:phiinrkhs} with corresponding operator $\GP$ with operator norm $K_\Phi$. Let the functions $Y_0$, $Y_1$ obey Assumption \ref{asmp:yinrkhs}, with bounded Hilbert space norm $K$ . Let $h : \cR \times \{0,1\} \rightarrow \R$ be an hypothesis, and assume that there exist $f^h_t \in \cH_r$ such that $h(r,t) = \left<f^h_t,k_r(r,\cdot)\right>_{\cH_r}$, such that $\|f^h_t\|_{\cH_r} \leq b$. Assume that $\epsilon_F$ and $\epsilon_{CF}$ are defined with respect to $L$ being the squared loss. Then:
\begin{align}\label{eq:thm_mmdA}
&\epsilon_{CF}(h,\Phi) \leq \nonumber\\
&\quad (1-u)  \epsilon^{t=1}_F(h,\Phi) + u  \epsilon^{t=0}_F(h,\Phi)  + \nonumber\\
& \quad 2 \left(K_\Phi^2 (K^2 +M^2) + b^2\right)  \cdot \text{MMD}(\pt_\Phi\, ,  \pc_\Phi),\nonumber\\
\end{align}
where $\epsilon_{CF}$ and $\epsilon_F$ use the squared loss.
\end{thmapplem}
\begin{proof}
We will apply Lemma \ref{lem:genA} with $\cF = {f \in \cH_r \otimes \cH_r \text{ s.t. } \|f\|_{\cH_r \otimes \cH_r} \leq 1}$. By Lemma \ref{lem:mmdbound}, we have that for $B_\Phi = 2\left(K_\Phi^2 (K^2+M^2) + b^2\right)$ and $L$ being the squared loss, $\frac{1}{B_\Phi} \lythr \in \cF$. Inequality \eqref{eq:thm_mmdA} then holds as a special case of Lemma \ref{lem:genA}.
\end{proof}

%\newpage
\begin{thmappthm}
Under the assumptions of Lemma \ref{thm:mmdA}, using the squared loss for $\epsilon_F$, we have:
\begin{align*}
&\epehe(h,\Phi) \leq \nonumber \\
&2 \epsilon^{t=0}_F(h,\Phi)  + 2\epsilon^{t=1}_F(h,\Phi)  - 4\sigma^2_Y +  \nonumber \\
&4 \left(K_\Phi^2 (K^2 +M^2) + b^2\right)  \cdot \text{MMD}( \pt_\Phi \, , \pc_\Phi) .
\end{align*}
\end{thmappthm}
\begin{proof}
Plug in the upper bound of Lemma \ref{thm:mmdA} into the upper bound of Theorem \ref{thm:indtausqloss}.
\end{proof}

\section{Algorithmic details}\label{sec:appmodel}

We give details about the algorithms used in our framework. 

%First, we restate Algorithm 1.
%
%\begin{algorithm}[tbp]
%\caption{CFR: Counterfactual regression with integral probability metrics}
%\label{alg:appmodel}
%\begin{algorithmic}[1]
%  \STATE \textbf{Input:} Factual sample $(x_1,t_1,y_1), \ldots , (x_n,t_n,y_n)$, scaling parameter $\alpha>0$, loss function $L\left(\cdot,\cdot\right)$, representation network $\Phi_{\bf{W}}$ with initial weights $\bf{W}$, outcome network $h_{\bf{V}}$ with initial weights $\bf{V}$, function family $F$ for IPM
%  \WHILE{not converged}
%    \STATE Sample $m$ control $\{(x_{i_j},0,y_{i_j})\}_{j=1}^m$ and $m'$ treated units $\{(x_{i_{k}},1,y_{i_{k}})\}_{k=m+1}^{m+m'}$
%    \STATE Calculate the gradient of the IPM term:\\
%    $g_1 = ${\small $\nabla_{\bf{W}} \; \text{IPM}_F(\{\Phi_{\bf{W}}(x_{i_j})\}_{j=1}^m, \{\Phi_{\bf{W}}(x_{i_{k}})\}_{k=m+1}^{m+m'} )$}
%    \STATE Calculate the gradients of the empirical loss: \\
%    $g_2 = \nabla_{\bf{V}}\frac{1}{m+m'} \sum_j L\left(h_{\bf{V}}(\Phi_{\bf{W}}(x_{i_j}),t_{i_j}), y_{i_j}\right)$ \\
%    $g_3 = \nabla_{\bf{W}}\frac{1}{m+m'}\sum_j L\left(h_{\bf{V}}(\Phi_{\bf{W}}(x_{i_j}),t_{i_j}), y_{i_j}\right)$
%    \STATE Obtain step size scalar or matrix $\eta$ with standard neural net methods e.g. RMSProp~\cite{tieleman2012lecture}
%    \STATE Update $\bf{W} \leftarrow \bf{W} - \eta(\alpha g_1 +  g_3 )$,  $\bf{V} \leftarrow \bf{V} - \eta g_2 $
%    \STATE Check convergence criterion
%  \ENDWHILE
%\end{algorithmic}
%\end{algorithm}

\subsection{Minimizing the Wasserstein distance}
In general, computing (and minimizing) the Wasserstein distance involves solving a linear program, which may be prohibitively expensive for many practical applications. \citet{cuturi2013sinkhorn} showed that an approximation based on entropic regularization can be obtained through the Sinkhorn-Knopp matrix scaling algorithm, at orders of magnitude faster speed. Dubbed Sinkhorn distances, the approximation is computed using a fixed-point iteration involving repeated multiplication with a kernel matrix $K$. We can use the algorithm of \citet{cuturi2013sinkhorn} in our framework. See Algorithm~\ref{alg:wassgrad} for an overview of how to compute the gradient $g_1$ in Algorithm~\ref{alg:model}. When computing $g_1$, disregarding the gradient $\nabla_{\bf W} T^*$ amounts to minimizing an upper bound on the Sinkhorn transport. More advanced ideas for stochastic optimization of this distance have recently proposed by \citet{aude2016stochastic}, and might be used in future work.

\begin{algorithm}[tbp]
\caption{Computing the stochastic gradient of the Wasserstein distance}
\label{alg:wassgrad}
\begin{algorithmic}[1]
  \STATE \textbf{Input:} Factual  $(x_1,t_1,y_1), \ldots , (x_n,t_n,y_n)$, representation network $\Phi_{\bf{W}}$ with current weights by $\bf{W}$
  \STATE Randomly sample a mini-batch with $m$ treated and $m'$ control units $(x_{i_1},0,y_{i_1}), \ldots , $\\
  $(x_{i_m},0,y_{i_m}),  (x_{i_{m+1}},1,y_{i_{m+1}}), \ldots , (x_{i_{2m}},1,y_{i_{2m}}) $
   \STATE Calculate the $m \times m$ pairwise distance matrix between all treatment and control pairs $M(\Phi_{\bf{W}})$: \\
   $M_{kl}(\Phi) = \|\Phi_{\bf{W}}(x_{i_k}) - \Phi_{\bf{W}}(x_{i_{m+l}})\|$
    \STATE Calculate the approximate optimal transport matrix $T^*$ using Algorithm 3 of \citet{cuturi2014fast}, with input $M(\Phi_{\bf{W}})$
  \STATE Calculate the gradient:\\ $g_1 = \nabla_{\bf{W}} \left< T^*,M(\Phi_{\bf{W}})\right>$
  \vspace{0.2em}
\end{algorithmic}
\end{algorithm}

%For non-uniform populations $u := p(t=1) \neq 1/2$, our method can still be applied using the unnormalized version of the Wasserstein distance, using ideas presented by \cite{guittet2002extended,gramfort2015fast}, and explored much more deeply by \cite{chizat2015unbalanced}. The Sinkhorn transport can still be used to approximate the Wasserstein distance, but with small modifications.
%Let $\lambda$ and $\delta$ be parameters and $\tilde{M}$ the matrix
%$$
%\tilde{M} = \left[ \begin{array}{cc}
%M & \delta \\
%\delta & 0\end{array} \right]~.
%$$
%Then, define an $n_t+1$-dimensional vector $a$
%$$
%a = [u, ..., u, 1-u]^\top
%$$
%and an $n_c+1$-dimensional vector $b$
%$$
%b = [1-u, ..., 1-u, u]^\top
%$$
%where $n_t$ and $n_c$ are the number of treated and controls.
%Then, to get $T^*$, apply Algorithm 3 of \cite{cuturi2014fast} on $\tilde{M}, a$ and $b$.

While our framework is agnostic to the parameterization of $\Phi$, our experiments focus on the case where $\Phi$ is a neural network. For convenience of implementation, we may represent the fixed-point iterations of the Sinkhorn algorithm as a recurrent neural network, where the states $u_t$ evolve according to
$$
u_{t+1} = n_t ./ (n_c K (1./(u_t^\top K)^\top))~.
$$
Here, $K$ is a kernel matrix corresponding to a metric such as the euclidean distance, $K_{ij} = e^{-\lambda\|\Phi(x_i) - \Phi(x_j)\|_2}$, and $n_c, n_t$ are the sizes of the control and treatment groups. In this way, we can minimize our entire objective with most of the frameworks commonly used for training neural networks, out of the box.

\subsection{Minimizing the maximum mean discrepancy}

The MMD of treatment populations in the representation $\Phi$, for a kernel $k(\cdot,\cdot)$ can be written as,
\begin{align}
\text{MMD}_k(\{\Phi_{\bf{W}}(x_{i_j})\}_{j=1}^m, \{\Phi_{\bf{W}}(x_{i_{k}})\}_{k=m+1}^{m'} ) = \\
\frac{1}{m(m-1)}\sum_{j=1}^m \sum_{k=1,k\neq j}^{m} k(\Phi_{\bf{W}}(x_{i_j}), \Phi_{\bf{W}}(x_{i_k})) \\
+ \frac{2}{mm'}\sum_{j=1}^m \sum_{k=m}^{m+m'} k(\Phi_{\bf{W}}(x_{i_j}), \Phi_{\bf{W}}(x_{i_k})) \\
+ \frac{1}{m'(1-m')}\sum_{j=1}^m \sum_{k=m,k\neq j}^{m'} k(\Phi_{\bf{W}}(x_{i_j}), \Phi_{\bf{W}}(x_{i_k}))
\end{align}

%The unnormalized version, where the marginal probability of treatment $u \neq 1/2$, is
%\begin{align}
%\text{MMD}_k(\{\Phi_{\bf{W}}(x_{i_j})\}_{j=1}^m, \{\Phi_{\bf{W}}(x_{i_{k}})\}_{k=m+1}^{m'} ) = \\
%\frac{2u^2}{m(m-1)}\sum_{j=1}^m \sum_{k=1,k\neq j}^{m} k(\Phi_{\bf{W}}(x_{i_j}), \Phi_{\bf{W}}(x_{i_k})) \\
%+ \frac{4u(1-u)}{mm'}\sum_{j=1}^m \sum_{k=m}^{m+m'} k(\Phi_{\bf{W}}(x_{i_j}), \Phi_{\bf{W}}(x_{i_k})) \\
%+ \frac{2(1-u)^2}{m'(1-m')}\sum_{j=1}^m \sum_{k=m,k\neq j}^{m'} k(\Phi_{\bf{W}}(x_{i_j}), \Phi_{\bf{W}}(x_{i_k}))
%\end{align}

The linear maximum-mean discrepancy can be written as a distance between means. In the notation of Algorithm~\ref{alg:model},
$$
\text{MMD} = 2\left\| \frac{1}{m}\sum_{j=1}^m \Phi_{\bf{W}}(x_{i_j}) -\frac{1}{m'}\sum_{k=m+1}^{m'} \Phi_{\bf{W}}(x_{i_k})\right\|_2
$$
Let
$$
{\bf f}({\bf W}) = \frac{1}{m}\sum_{j=1}^m \Phi_{\bf{W}}(x_{i_j}) - \frac{1}{m'}\sum_{k=m+1}^{m+m'} \Phi_{\bf{W}}(x_{i_k})
$$
Then the gradient of the MMD with respect to $\bf{W}$ is,
$$
g_1 = 2 \frac{d \bf{f}(\bf{W})}{d\bf{W}} \frac{\bf{f}(\bf{W})}{\|\bf{f}(\bf{W})\|_2} ~.
$$

\section{Experimental details}\label{sec:appexp}

\subsection{Hyperparameter selection}
Standard methods for hyperparameter selection, such as cross-validation, are not generally applicable for estimating the PEHE loss since only one potential outcome is observed (unless the outcome is simulated). For real-world data, we may use the observed outcome $y_{j(i)}$ of the nearest neighbor $j(i)$ to $i$ in the opposite treatment group, $t_{j(i)} = 1 - t_i$ as surrogate for the counterfactual outcome. We use this to define a nearest-neighbor approximation of the PEHE loss, $\epehenn(f) = \frac{1}{n}\sum_{i=1}^n \left((1-2t_i)(y_{j(i)} - y_i) - (f(x_i,1) - f(x_i,0))\right)^2~$. On IHDP, we use the objective value on the validation set for early stopping in CFR, and $\epehenn(f)$ for hyperparameter selection. On the Jobs dataset, we use the policy risk on the validation set.

See Table~\ref{tbl:hypparams} for a description of hyperparameters and search ranges.

\begin{table}[t!]
  \caption{\label{tbl:hypparams}Hyperparameters and ranges.}
  \begin{center}
      \begin{tabular}{ll}
        Parameter & Range \\ \hline
        Imbalance parameter, $\alpha$ & $\{10^{k/2}\}_{k=-10}^6$ \\
        Num. of representation layers & $\{1,2,3\}$ \\
        Num. of hypothesis layers & $\{1,2,3\}$ \\
        Dim. of representation layers & $\{20, 50, 100, 200\}$ \\
        Dim. of hypothesis layers & $\{20, 50, 100, 200\}$ \\
        Batch size & $\{100, 200, 500, 700\}$ \\
        \hline
      \end{tabular}
    \end{center}
\end{table}

\subsection{Learned representations}
Figure~\ref{fig:reps} show the representations learned by our CFR algorithm.

\begin{figure*}[]
  \centering
  \begin{subfigure}[b]{0.3\textwidth}
    \centering
    \includegraphics[width=0.90\columnwidth]{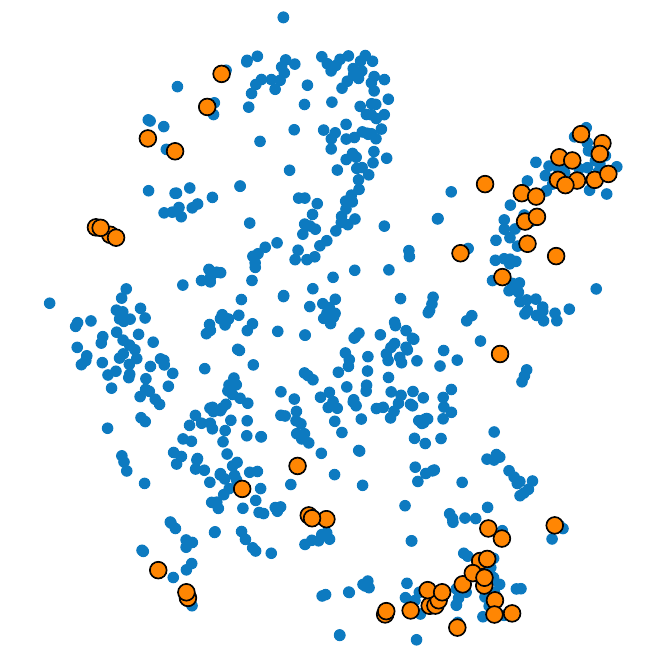}
    \caption{Original data}
  \end{subfigure}
  \begin{subfigure}[b]{0.3\textwidth}
    \centering
    \includegraphics[width=0.90\columnwidth]{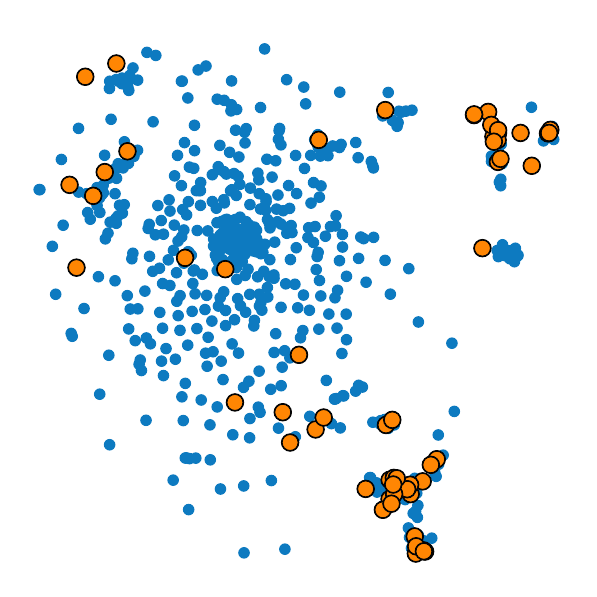}
    \caption{Linear MMD}
  \end{subfigure}
  \begin{subfigure}[b]{0.3\textwidth}
    \centering
    \includegraphics[width=0.90\columnwidth]{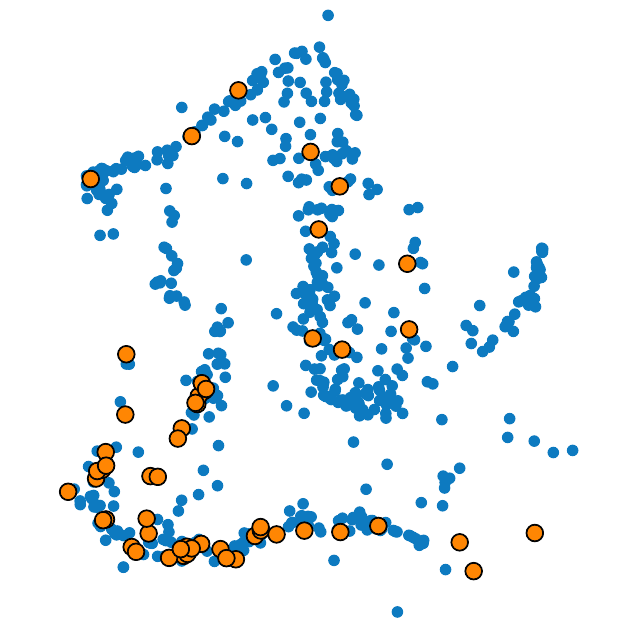}
    \caption{Wasserstein}
  \end{subfigure}
  \caption{\label{fig:reps}t-SNE visualizations of the balanced representations of IHDP learned by our algorithms CFR, CFR MMD and CFR Wass. We note that the nearest-neighbor like quality of the Wasserstein distance results in a strip-like representation, whereas the linear MMD results in a ball-like shape in regions where overlap is small.}
\end{figure*}

\subsection{Absolute error for increasingly imbalanced data}
Figure~\ref{fig:ihdp_abs_imb} shows the results of the same experiment as Figure 2 of the main paper, but in absolute terms.
\begin{figure}[t]
  \centering
  \includegraphics[width=0.90\columnwidth]{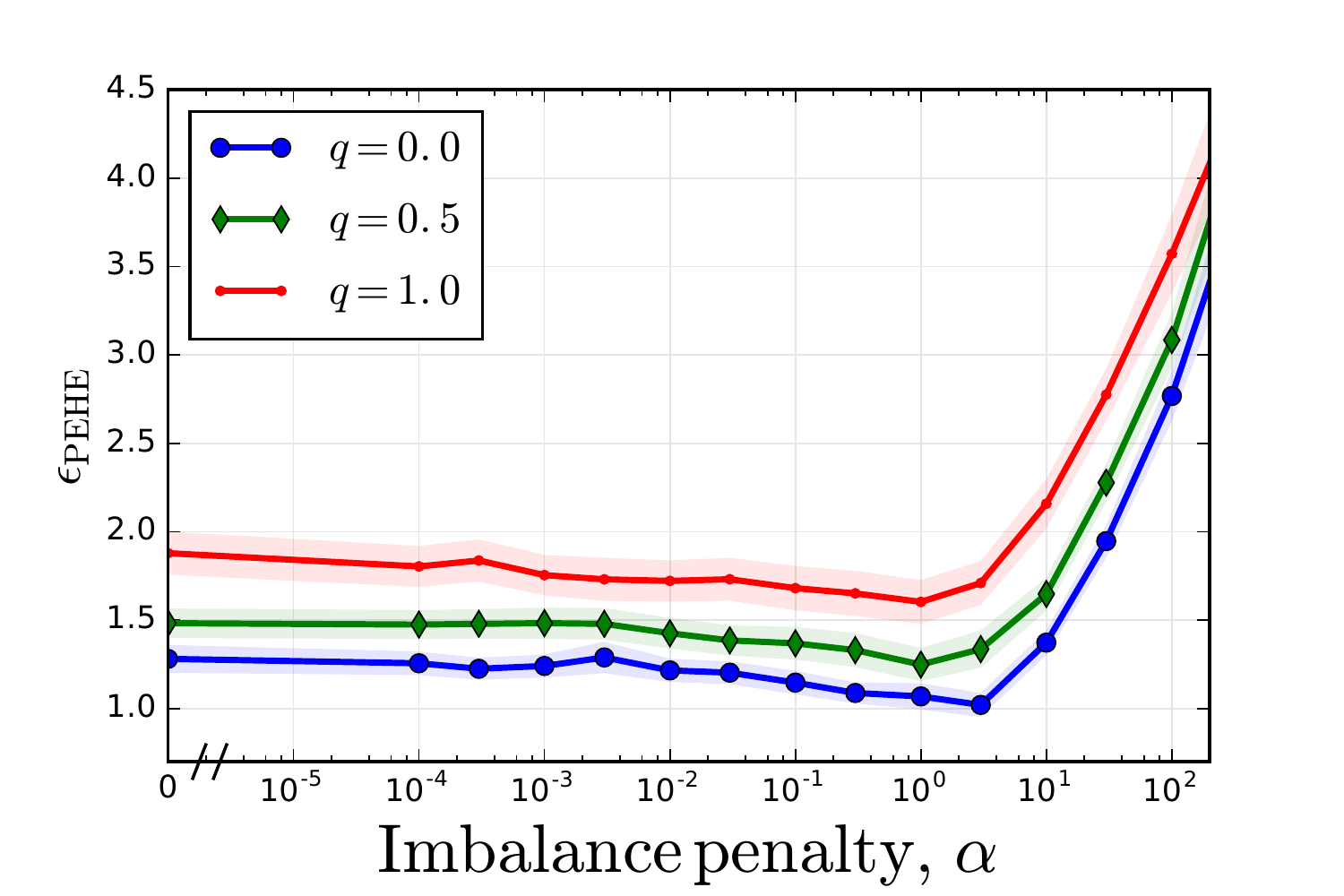}\vspace{-1em}
  \caption{\label{fig:ihdp_abs_imb}Out-of-sample error in estimated ITE, as a function of IPM regularization parameter for CFR Wass, on 500 realizations of IHDP, with high ($q=1$), medium and low (artificial) imbalance between control and treated. }
  \vspace{-1em}
\end{figure}

\end{document}